\documentclass[11pt]{article}

\usepackage[utf8]{inputenc}
\usepackage[T1]{fontenc}
\usepackage[margin=1in]{geometry}
\usepackage{authblk}
\usepackage[round,authoryear]{natbib}
\usepackage{microtype}
\usepackage{amsmath,amssymb,amsthm,mathtools}
\usepackage{pifont}
\usepackage{xcolor}
\usepackage{algorithm}
\usepackage{algpseudocode}
\usepackage{tikz}
\usetikzlibrary{arrows.meta,calc}
\usepackage{titletoc}
\usepackage{graphicx}
\usepackage{caption}
\usepackage{enumitem}
\usepackage{xurl}
\usepackage{amsmath,amsfonts,bm}

\def\1{\bm{1}}

\def\ve{{\bm{e}}}

\def\vx{{\bm{x}}}
\def\vy{{\bm{y}}}
\def\vz{{\bm{z}}}

\def\mP{{\bm{P}}}
\def\mQ{{\bm{Q}}}

\DeclareMathAlphabet{\mathsfit}{\encodingdefault}{\sfdefault}{m}{sl}
\SetMathAlphabet{\mathsfit}{bold}{\encodingdefault}{\sfdefault}{bx}{n}

\newcommand{\E}{\mathbb{E}}
\newcommand{\Prob}{\mathbb{P}}

\newcommand{\KL}{D_{\mathrm{KL}}}
\newcommand{\TV}{d_{\mathrm{TV}}}
\newcommand{\dHam}{d_{\mathrm{Ham}}}
\newcommand{\DBr}{D_{\mathrm{Br}}}

\DeclareMathOperator{\supp}{supp}
\usepackage{hyperref}
\usepackage{setspace}

\hypersetup{
  pdftitle={Uniform Discrete Diffusion Models are Minimax Optimal for Estimating Distributions with Small Effective Support Size},
  pdfauthor={Dongsun Yoon, Saptarshi Chakraborty},
  colorlinks=true,
  linkcolor=blue,
  citecolor=blue,
  urlcolor=teal,
  bookmarksnumbered=true
}

\theoremstyle{plain}
\newtheorem{theorem}{Theorem}
\newtheorem{proposition}[theorem]{Proposition}
\newtheorem{lemma}[theorem]{Lemma}
\newtheorem{corollary}[theorem]{Corollary}

\theoremstyle{definition}
\newtheorem{definition}[theorem]{Definition}
\newtheorem{assumption}[theorem]{Assumption}
\theoremstyle{remark}

\title{Uniform Discrete Diffusion Models are Minimax Optimal for Estimating Distributions with Small Effective Support Size}
\author[1]{Dongsun Yoon\thanks{Email: \texttt{dsyoon@umich.edu}}}
\author[1]{Saptarshi Chakraborty\thanks{Email: \texttt{saptarsc@umich.edu}}}
\affil[1]{Department of Statistics, University of Michigan}
\date{}

\begin{document}
\maketitle

\begin{abstract}
Discrete diffusion models have emerged as a practically successful framework for generative modeling on discrete product spaces, yet their statistical generalization properties remain poorly understood.
Discrete real-world data such as text or biological sequences often concentrate on a small fraction of the astronomically large ambient space because of semantic or physical constraints, but existing bounds fail to capture this distributional structure and instead scale with the size of the ambient space, giving rise to almost vacuous error bounds.
We address this gap for uniform discrete diffusion, one of the two dominant discrete diffusion paradigms alongside masking diffusion, by deriving statistical guarantees governed by the \emph{effective support size} $\mathfrak s_n(P_0)$, a sample-size-dependent measure of distributional complexity.
Given $n$ independent and identically distributed (i.i.d.) samples from an unknown data distribution $P_0$ on $[K]^d$, we show that, with appropriate choices of network size and hyperparameters, the expected total variation (TV) loss scales as $\mathcal{O}\left(\sqrt{\mathfrak{s}_n(P_0)/n}\right)$, while the expected Kullback--Leibler (KL) divergence is bounded by \(\mathcal{O}\left(\frac{1}{n}\mathfrak{s}_n(P_0)\log(eK^d/\mathfrak{s}_n(P_0))\log n\right)\).
Furthermore, we show that the TV rate is minimax optimal and that the KL rate is minimax optimal up to a factor of $\log n$.
Together, these upper and lower bounds show that uniform discrete diffusion successfully avoids the curse of dimensionality for distributions with small effective support size: the TV error rate depends on the ambient state-space size only through $\mathfrak{s}_n(P_0)$, while the corresponding KL rate incurs only an additional logarithmic dependence on the ambient state-space size.
\end{abstract}
% \endgroup
\par

\section{Introduction}
\label{sec:introduction}

Diffusion models have become a central paradigm for generative modeling \citep{sohl-dickstein_deep_2015,ho_denoising_2020,song_score-based_2021}.
Extensions to discrete state spaces have carried this paradigm to categorical domains such as text, graphs, molecules, and combinatorial optimization \citep{he_diffusionbert_2023,hoogeboom_argmax_2021,vignac_digress_2023,sun_difusco_2023}.
\citet{austin_structured_2021} introduced structured categorical corruption kernels in discrete time, and \citet{campbell_continuous_2022} cast both the forward and reverse processes as continuous-time Markov chains (CTMCs).
Score-based approaches that estimate probability ratios through score entropy have since achieved strong empirical performance in language modeling \citep{lou_discrete_2024}.
Two forward processes dominate: \emph{uniform diffusion}, in which each coordinate evolves toward uniformity, and \emph{masking diffusion}, in which each coordinate is absorbed into a special mask state.
They induce distinct reverse-time dynamics and hence different statistical and algorithmic properties \citep{dmitriev_efficient_2026,zhang_generalization_2026}.
In this work, we focus on coordinatewise uniform diffusion on the product space $[K]^d$.

The empirical success of discrete diffusion has spurred substantial theoretical work.
Most of this work concerns the convergence of reverse-time samplers \citep{campbell_continuous_2022,chen_uniformization_2024,ren_stochastic_2025,zhang_convergence_2025,pham_discrete_2025,liang_discrete_2025,dmitriev_efficient_2026}, particularly fixed-grid $\tau$-leaping and its variants.
For uniform discrete diffusion, representative results \citep{liang_discrete_2025,dmitriev_efficient_2026} take the form
\[
D_{\mathrm{KL}}(P_\delta\|Q_\delta)
\lesssim
\varepsilon_{\mathrm{score}}
+ e^{-T}d\log K
+ \mathcal R_{\mathrm{disc}}.
\]
Here, $P_\delta$ is the forward marginal at the early-stopping time $\delta$, and $Q_\delta$ is the output distribution of the sampler initialized from the uniform distribution at time $T$.
The first term $\varepsilon_{\mathrm{score}}$ bounds the interval-weighted score-entropy error on the sampling grid, and the remaining terms account for initialization and time discretization.
These guarantees take an $\varepsilon_{\mathrm{score}}$-accurate score estimate as given and do not establish that such an estimate can be learned.
It therefore remains unclear how they translate into finite-sample bounds on the distribution-estimation risk when the score is learned from $n$ training samples.

Without exploiting structure in the data distribution, finite-sample guarantees face a fundamental obstacle: the size of the ambient space.
For arbitrary distributions on $[K]^d$, the minimax expected risk is of order $\Theta\left(\min\{1,\sqrt{K^d/n}\}\right)$ under TV loss \citep{han_minimax_2015,kamath_learning_2015} and $\Theta\left(\log(1+K^d/n)\right)$ under KL loss \citep{ye_instance_2025, mourtada_estimation_2026}.
Consistent estimation over the unrestricted class therefore requires $K^d/n\to 0$, a condition far from realistic.
The Score Entropy Discrete Diffusion (SEDD) language experiments \citep{lou_discrete_2024}, for instance, use a vocabulary size $K=50{,}257$ and sequence length $d=1024$. The resulting ambient space of size $50{,}257^{1024}$ is astronomically larger than any feasible training set.

The finite-sample statistical theory of discrete diffusion that has recently begun to emerge does not escape this barrier \citep{wakasugi_state_2025,cho_minimax_2026,srikanth_sample_2026,zhang_generalization_2026}.
The KL bounds of \citet{wakasugi_state_2025} and \citet{cho_minimax_2026} are of order $\widetilde{\mathcal O}(K^d/n)$, and the TV bounds of \citet{zhang_generalization_2026} grow exponentially in $d$, so all become vacuous when $n\ll K^d$.
Refinements that avoid this dependence either rely on assumptions on $P_0$ that are unverifiable in practice, such as fast spectral decay \citep{wakasugi_state_2025}, or depend on quantities that can themselves be as large as $K^d$.
For example, the bound of \citet{srikanth_sample_2026} scales polynomially with $\max_{d_{\mathrm{Ham}}(\vx,\vy)=1} P_0(\vy)/P_0(\vx)$, which is enormous when a single-token substitution sharply reduces sequence probability.

Such pessimistic guarantees are at odds with the structure of real data.
Text and biological sequences concentrate on a vanishing fraction of $[K]^d$ because of semantic or physical constraints, so the ambient cardinality greatly overstates their statistical complexity.
A natural first refinement replaces $K^d$ by the support size $|\supp(P_0)|$.
However, the support can itself be comparable to $K^d$ even when most of the mass lies on far fewer states.
Moreover, at sample size $n$, states with probability well below $1/n$ are unlikely to be observed and should not each count as a full unit of complexity.

We therefore measure complexity by the \emph{effective support size}
\[
\mathfrak{s}_n(P_0):=\sum_{\vx\in[K]^d}\min\{nP_0(\vx),1\}.
\]
Each state with probability at least $1/n$ counts as one effective state, while rarer states contribute according to their expected counts in $n$ samples.
Up to universal constants, $\mathfrak{s}_n(P_0)$ is the expected number of distinct states in a sample of size $n$ (Lemma~\ref{lem:D_n_s_n}).
It satisfies $\mathfrak{s}_n(P_0)\le\min\{n,|\supp(P_0)|\}$ and can be far smaller than both. Under Zipfian decay $p_{(j)}\asymp j^{-\alpha}$, for instance, $\mathfrak{s}_n(P_0)\asymp\min\{|\supp(P_0)|,n^{1/\alpha}\}$ (Section~\ref{subsec:statistical-setup}).
The notion of effective support size has been used in the classical literature on occupancy counts, unseen species, and missing mass \citep{good_population_1953,karlin_central_1967,benhamou_concentration_2017}.
Related occupancy-based notions have also yielded distribution-dependent guarantees for discrete distribution estimation \citep{falahatgar_power_2017,mourtada_estimation_2026}.

This leads to the central question of the paper:
\begin{quote}
\emph{Can discrete diffusion models attain distribution-estimation rates that adapt to $\mathfrak{s}_n(P_0)$ and remain informative even when $n\ll K^d$? }
\end{quote}

\paragraph{Summary of contributions.}
We answer this question affirmatively for uniform discrete diffusion.

\begin{itemize}[leftmargin=10pt]
\item \textbf{Adaptive guarantees for uniform discrete diffusion.}
We establish end-to-end finite-sample guarantees for uniform discrete diffusion with a learned score. They account for score estimation with ReLU networks as well as initialization and time-discretization errors, and they apply to any fixed-grid sampler satisfying a standard score-entropy guarantee, including $\tau$-leaping.
The expected TV loss is $\mathcal O\big(\sqrt{\mathfrak{s}_n(P_0)/n}\big)$, with number of sampling steps and network depth, width, and sparsity polynomial in $K$, $d$, and $n/\mathfrak{s}_n(P_0)$.
With early stopping at $\delta=1/n$, the expected KL divergence is $\mathcal O\big(\frac 1n \mathfrak{s}_n(P_0)\log(eK^d/\mathfrak{s}_n(P_0))\log n\big)$.
Both bounds depend on $K^d$ at most logarithmically and remain informative when $n\ll K^d$.
The KL guarantee carries particular weight: the empirical distribution has infinite KL risk whenever some states go unobserved, so attaining it requires generalizing beyond the training samples.

\item \textbf{Matching minimax lower bounds.}
For $2\le s_n\le \min\{n,K^d\}$, let $\mathcal P_{s_n}:=\{P_0\in\Delta_{K^d}: \mathfrak{s}_n(P_0)\le s_n\}$. We establish the non-asymptotic minimax rates
\[
\begin{aligned}
\inf_{\widetilde P}\sup_{P_0\in\mathcal P_{s_n}}
\E_{P_0}\TV(P_0,\widetilde P)
&\asymp \sqrt{\frac{s_n}{n}},\qquad
\inf_{\widetilde P}\sup_{P_0\in\mathcal P_{s_n}}
\E_{P_0}\KL(P_0\Vert\widetilde P)
&\asymp \frac{s_n\log(eK^d/s_n)}{n}.
\end{aligned}
\]
Hence, uniform discrete diffusion is minimax optimal under TV loss and optimal up to a $\log n$ factor under KL loss.
When $s_n=\min\{n,K^d\}$, \(\mathcal P_{s_n}=\Delta_{K^d}\) holds, and our rates reduce to the classical minimax rates for arbitrary distributions on \(K^d\) states.
Thus, our result extends the classical unrestricted theory, while yielding potentially much sharper rates for distributions with small effective support size.
\par\smallskip
\noindent
\begin{minipage}[t]{0.55\linewidth}
\vspace{-15pt}
Allowing \(K\) and \(d\) to vary with \(n\), these rates also characterize when consistent estimation is possible.
Under TV loss, this holds if and only if $s_n/n\to0$, and under KL loss, if and only if $(s_n/n)\log(eK^d/s_n)\to0$.
The gap between the two conditions yields three regimes (Figure~\ref{fig:asymptotic_regimes}): consistent estimation is impossible, possible under TV loss only, or possible under both losses.
\end{minipage}\hfill%
\begin{minipage}[t]{0.42\linewidth}
\vspace{-40pt}
\centering
\resizebox{\linewidth}{!}{%
\begin{tikzpicture}[
x=0.95cm, y=0.82cm, >=Latex, font=\small ]

\def\xcut{1.333} \def\xmax{5.833} \def\ymax{4.7}

\def\curve#1{1.00 + 0.09*(exp(0.80*(#1-\xcut))-1)}

\begin{scope}
\clip (0,0) rectangle (\xmax,\ymax);

\fill[red!18] (0,0) rectangle (\xcut,\ymax);

\fill[green!18] (\xcut,0) -- (\xmax,0) -- plot[ domain=\xmax:\xcut, samples=120, smooth ] (\x,{\curve{\x}}) -- cycle;

\fill[yellow!35] (\xcut,\ymax) -- (\xmax,\ymax) -- plot[ domain=\xmax:\xcut, samples=120, smooth ] (\x,{\curve{\x}}) -- cycle;

\draw[thick] (\xcut,0) -- (\xcut,\ymax);

\draw[thick] plot[ domain=\xcut:\xmax, samples=120, smooth ] (\x,{\curve{\x}});
\end{scope}

\draw[->,thick] (0,0) -- (\xmax+0.35,0) node[right] {$n/s_n$};

\draw[->,thick] (0,0) -- (0,\ymax+0.35) node[above] {$K^d/s_n$};

\node[below=4pt] at (0.667,0) {$n/s_n = O(1)$};

\node[below=4pt] at (3.783,0) {$n/s_n \to \infty$};

\node[align=center] at (0.667,2.25)
{
TV: \ding{55}\\[-1mm] KL: \ding{55} };

\node[align=center] at (3.633,3.8)
{
TV: \ding{51}\\[-1mm] KL: \ding{55} };

\node[align=center] at (3.183,2.55)
{
$\displaystyle \frac{s_n}{n}\log\!\left(\frac{eK^d}{s_n}\right) \not\to 0$ };

\node[align=center] at (5.133,1.6)
{
TV: \ding{51}\\[-1mm] KL: \ding{51} };

\node[align=center] at (3.983,0.63)
{
$\displaystyle \frac{s_n}{n}\log\!\left(\frac{eK^d}{s_n}\right)\to 0$ };

\end{tikzpicture}%
}

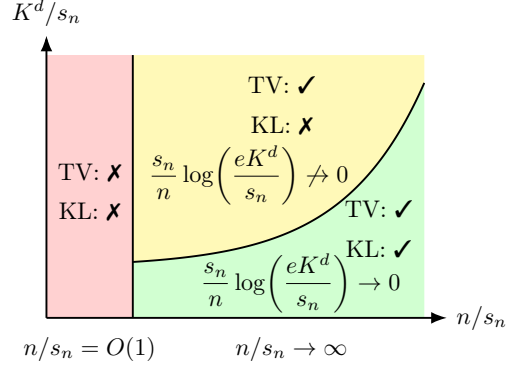
\captionof{figure}{Asymptotic statistical regimes for estimating discrete distributions on $[K]^d$ with $n$ i.i.d.\ observations.}
\label{fig:asymptotic_regimes}
\end{minipage}
\par\medskip

\end{itemize}

\section{Background and Problem Setting}
\label{sec:problem-setting}

\paragraph{Notations.}
We use plain uppercase letters such as $P$ and $Q$ for probability distributions and bold uppercase letters such as $\mP$ and $\mQ$ for transition and generator matrices.
For a positive integer $K$, write $[K]:=\{1,\ldots,K\}$, and let $\Delta_{K^d}$ denote the probability simplex on $[K]^d$.
For a probability distribution $P$ on $[K]^d$, let $\supp(P):=\{\vx:P(\vx)>0\}$ denote its support, and write $\E_P$ for expectation under $P$.
For probability distributions $P$ and $Q$ on $[K]^d$, the Total Variation (TV) distance and Kullback-Leibler divergence are defined as, 
\(\TV(P,Q):=\frac12\sum_{\vx \in [K]^d}|P(\vx)-Q(\vx)|, \,
  \KL(P\Vert Q):=\sum_{\vx \in [K]^d} P(\vx)\log\left(P(\vx)/Q(\vx)\right),\)
where, we use the conventions $0\log(0/q):=0$ and $p\log(p/0):=\infty$ for $p>0$.
For nonnegative quantities $u_n$ and $v_n$, which are functions of $n$, we write,  $u\lesssim v$ if $u_n\le Cv_n $ for a universal constant $C>0$, independent of $n$. We write $u_n\asymp v_n$ if both $u_n \lesssim v_n$ and $v_n \lesssim u_n$.
For $\vx=(x^1,\ldots,x^d),\vy=(y^1,\ldots,y^d)\in[K]^d$,  the Hamming distance between $\vx$ and $\vy$ is defined as,  $\dHam(\vx,\vy):=\sum_{i=1}^d \1\{x^i\neq y^i\}$, where $\1(\cdot)$ denotes the indicator function.
For $u,v>0$, define the Bregman divergence generated by $f(u):=u\log u$ as $\DBr(u,v):=f(u)-f(v)-f'(v)(u-v)=u\log(u/v)-u+v$.
For $a\in[K]$, let $\ve_a\in\{0,1\}^K$ be the $a$-th standard basis vector, and define the coordinatewise one-hot encoding $\ve(\vx):=(\ve_{x^1}^{\top},\ldots,\ve_{x^d}^{\top})^{\top}\in\{0,1\}^{dK}$.
$\delta_{\vx}$ denotes the Dirac distribution at $\vx$.
For $\ell\le r$, define $\operatorname{clip}_{[\ell,r]}(u):=\min\{\max\{u,\ell\},r\}$.
All logarithms are with respect to (w.r.t.) the natural base.

\subsection{Statistical Setup and Effective Support Size}
\label{subsec:statistical-setup}

We consider distribution estimation on the product space $[K]^d$ for $K,d \ge 2$, where $K$ can be interpreted as the vocabulary size and $d$ is the sequence length. Let $P_0\in\Delta_{K^d}$ be an unknown data distribution, and suppose that we observe training $n$ i.i.d. samples $X_0^{(1)},\ldots,X_0^{(n)}\stackrel{\mathrm{i.i.d.}}{\sim}P_0$ for $n \ge 2$.
A distribution estimator $\widetilde P \in \Delta_{K^d}$ is any (possibly randomized) data-dependent estimator of $P_0$.
We study its expected total variation and Kullback--Leibler risks,
\(\E_{P_0}\TV(P_0,\widetilde P) \; \text{and} \;
  \E_{P_0}\KL(P_0\Vert\widetilde P).
\)

\paragraph{Effective support size.} As a measure of the complexity of $P_0$ relative to a sample of size $n$, we use its \emph{effective support size} \citep{mourtada_estimation_2026,falahatgar_power_2017},
\[
\mathfrak{s}_n(P_0):=\sum_{\vx\in[K]^d}\min\{nP_0(\vx),1\} = \left| \left\{\vx \in [K]^d: P_0(\vx) \ge 1/n\right\}\right| + n \sum_{\vx \, :\, P_0(\vx) < 1/n } P_0(\vx),
\]
where states above the sampling resolution $1/n$ each contribute a unit, while states below it contribute only their rescaled aggregate mass. Thus $\mathfrak{s}_n(P_0)$ is nondecreasing in $n$, satisfies $\mathfrak{s}_n(P_0) \le \min\{n, |\operatorname{supp}(P_0)|\}$, and recovers the exact support size as $n \to \infty$, i.e., $\lim_{n \to \infty}\mathfrak{s}_n(P_0)  = \mathfrak{s}_\infty(P_0) := \left| \left\{\vx \in [K]^d: P_0(\vx) > 0\right\}\right|$. This refines the crude sparsity measure $|\operatorname{supp}(P_0)|$ into one calibrated to the sample size, and is sensitive not just to $|\operatorname{supp}(P_0)|$ but to the decay of the point masses. Writing $p_{(1)}\ge\cdots\ge p_{(|\supp(P_0)|)}>0$ for the sorted nonzero masses of $P_0$, Lemma~\ref{lem:s_n_behavior} gives, for $n \ge 2$,
\begin{equation}
\label{eq:s_n_behavior}
\mathfrak{s}_n(P_0)\asymp
\begin{cases}
  \min\{|\supp(P_0)|,n\},
  & \text{ if }p_{(j)}\asymp|\supp(P_0)|^{-1},\\[2pt]
  \min\{|\supp(P_0)|,n^{1/\alpha}\},
  & \text{ if } p_{(j)}\asymp j^{-\alpha},\\[2pt]
  \min\{|\supp(P_0)|,\log n\},
  & \text{ if } p_{(j)}\asymp e^{-cj},
\end{cases}
\end{equation}
uniformly over $1 \le j \le |\supp(P_0)|$, with $\alpha>1$, $c>0$ fixed and implicit constants independent of $n$ and $|\supp(P_0)|$. These regimes span the full range of behavior: linear growth under near-uniform mass, sublinear polynomial growth $n^{1/\alpha}$ under Zipfian decay, and merely logarithmic growth under geometric decay. It is this decay-dependent behavior, rather than the ambient size $K^d$ or the raw support size $|\supp(P_0)|$, that governs the finite-sample rates established below.

For $1\le s_n\le \min\{n,K^d\}$, we define the effective support class
\(\mathcal P_{s_n}=\mathcal P_{s_n}(K^d,n):=\{P_0\in\Delta_{K^d}:\mathfrak{s}_n(P_0)\le s_n\}.\)
The corresponding minimax risks in the TV and KL are
\[
\inf_{\widetilde P}\sup_{P_0\in\mathcal P_{s_n}}\E_{P_0}\TV(P_0,\widetilde P)
  \qquad\text{and}\qquad \inf_{\widetilde P}\sup_{P_0\in\mathcal P_{s_n}}
  \E_{P_0}\KL(P_0\Vert\widetilde P).
\]
The goal of this paper is to study the behavior of the above minimax risks and to understand under what model and hyperparameter choices discrete diffusion models can achieve this optimal rate.
\subsection{Uniform Discrete Diffusion}
\label{subsec:uniform-discrete-diffusion}
The uniform discrete diffusion consists of a forward noising process, its score-based time reversal denoising, a fixed-grid score estimator, and a reverse sampler.

\paragraph{Forward process.}
The forward process is a continuous-time Markov chain (CTMC) on $[K]^d$ with independently evolving coordinates.
The time-homogeneous generator of each coordinate is $\mQ^{\mathrm{coord}}(a,b):=1/K-\1\{a=b\}$ for every $a,b\in[K]$, which means that each coordinate jumps to another state with rate $1/K$.
Following \citet{campbell_continuous_2022}, the induced generator $\mQ$ on $[K]^d$ is
\[ \mQ(\vx,\vy)=   K^{-1} \1\left\{ \dHam(\vx,\vy)=1\right\} -d K^{-1}(K-1) \1 \left\{\vx=\vy\right\}.\] 
The one-coordinate transition kernel is given by  \(\mP_t^{\mathrm{coord}}(a,b) =e^{-t}\1\{a=b\}+(1-e^{-t})/K,\) and coordinatewise independence gives \(\mP_t(\vx,\vy)
=\prod_{i=1}^d \mP_t^{\mathrm{coord}}(x^i,y^i)\) and \(\mP_t=e^{t\mQ}.\)
Writing $P_t:=P_0\mP_t$, symmetry of $\mQ$ implies that the uniform distribution $\pi(\vx):=K^{-d}$ is stationary.
Since the chain is irreducible on the finite state space $[K]^d$, $\pi$ is the unique stationary distribution and $P_t\to\pi$ as $t\to\infty$ \citep[Theorem~3.6.2]{norris_markov_1997}.

\paragraph{Reverse process and discrete score.}
Fix a terminal time $T>0$, and let $u\in[0,T]$ denote reverse time.
Initialized from $P_T$, the time reversal of the forward process is a time-inhomogeneous CTMC whose off-diagonal rates, for $u\in[0,T)$, are
\(\overleftarrow{\mQ}_u^P(\vx,\vy) = \mQ(\vy,\vx)\frac{P_{T-u}(\vy)}{P_{T-u}(\vx)}, \; \vx\neq\vy. \)
The diagonal entries are chosen so that each row sums to zero.
With the forward generator fixed, the probability ratio $P_{T-u}(\vy)/P_{T-u}(\vx)$ determines the reverse generator.
This motivates defining the true discrete score, for $t>0$, as
\(\sigma_t^\star(\vy,\vx):=\frac{P_t(\vy)}{P_t(\vx)}.
\)
This ratio plays a role analogous to the continuous score $\nabla\log p_t$, which governs the reverse drift in continuous diffusion models.
For our coordinatewise uniform forward process, the off-diagonal reverse rates become
\[\overleftarrow{\mQ}_u^P(\vx,\vy) =
\begin{cases}
  K^{-1}\sigma_{T-u}^\star(\vy,\vx),
  & \dHam(\vx,\vy)=1,\\
  0, & \dHam(\vx,\vy)\ge2,
\end{cases}
\qquad \vx\neq\vy.
\]
For sufficiently large $T$, generation starts from the tractable stationary distribution $\pi$ in place of $P_T$ and simulates the reverse process using a learned approximation of $\sigma_t^\star$.

\paragraph{Fixed-grid score estimation.}
Although the reverse process evolves continuously in time, a practical sampler evaluates the score at only finitely many grid-points. In this paper, we choose a deterministic forward-time grid $\mathcal T:=\{\delta=t_0<t_1<\cdots<t_N=T\}$,
where $\delta\ge0$ is the early-stopping time and scores are estimated only at $t_1,\ldots,t_N$.
For a positive candidate score $\sigma_t$, the score-entropy loss at time $t>0$ is given by, 
\[ \mathcal L_t^{\mathrm{SE}}(\sigma;P_0):=\E_{X_t\sim P_t}\left[ \sum_{\vy:\mQ(\vy,X_t)>0}\mQ(\vy,X_t) \DBr\left(\sigma_t^\star(\vy,X_t),\sigma_t(\vy,X_t)\right) \right].\]
This loss is nonnegative and minimized by the true score $\sigma_t^\star$ on all pairs $(\vy,\vx)$ satisfying $\dHam(\vy,\vx)=1$ \citep{lou_discrete_2024}.
We aggregate the score error over the grid using the interval lengths as weights:
\(\mathcal L_{\mathcal T}^{\mathrm{SE}}(\sigma;P_0)
:=\sum_{k=1}^N(t_k-t_{k-1})
\mathcal L_{t_k}^{\mathrm{SE}}(\sigma;P_0).\)
Since $\mathcal L_t^{\mathrm{SE}}$ depends on the intractable marginal $P_t$, we instead use its denoising form. Denoting the true conditional score by $\sigma_{t\mid0}^\star(\vy,\vx\mid\vx_0):=\mP_t(\vx_0,\vy)/\mP_t(\vx_0,\vx)$, the denoising score-entropy loss is constructed as: 
\begin{align*}
    & \mathcal L_t^{\mathrm{DSE}}(\sigma;P_0) := &\E\left[
\sum_{\vy:\mQ(\vy,X_t)>0}\mQ(\vy,X_t)
\left(\sigma_t(\vy,X_t)
-\sigma_{t\mid0}^\star(\vy,X_t\mid X_0)\log\sigma_t(\vy,X_t)\right)
\right],
\end{align*}
where in the above expectation, $X_0\sim P_0, \text{ and } X_t|X_0\sim\mP_t(X_0,\cdot)$.
As shown in Lemma~\ref{lem:se_dse_equivalent}, the score-entropy and denoising score-entropy losses differ only by a term independent of $\sigma$ and therefore have the same minimizer.
Accordingly, we define the denoising score-matching loss as
\(\mathcal L_{\mathcal T}^{\mathrm{DSE}}(\sigma;P_0)
:=\sum_{k=1}^N(t_k-t_{k-1})
\mathcal L_{t_k}^{\mathrm{DSE}}(\sigma;P_0).\)
During training, since $P_0$ is unknown, it is substituted by the empirical distribution $\widehat P_0:=n^{-1}\sum_{i=1}^n\delta_{X_0^{(i)}}$, yielding the empirical training objective 
\begin{equation}
    \widehat{\mathcal L}_{\mathcal T}^{\mathrm{DSE}}(\sigma)
:=\mathcal L_{\mathcal T}^{\mathrm{DSE}}(\sigma;\widehat P_0). \label{trobj}
\end{equation}
\paragraph{Neural-network score class.}
For positive integers $L,W,S$ and $B>0$, let $\mathcal G(L,W,S,B)$ denote the class of scalar-output ReLU networks on $\mathbb R^{1+2dK}$ with depth at most $L$, width at most $W$, at most $S$ nonzero parameters, and all weights and biases bounded in absolute value by $B$.
For $G\in\mathcal G(L,W,S,B)$ and $M\ge1$, define
\[
\sigma_{t_k}^G(\vy,\vx)
:=\operatorname{clip}_{[M^{-1},M]}
\left(G\left(k/N,\ve(\vy),\ve(\vx)\right)\right),
\qquad k\in[N],\quad \vx,\vy\in[K]^d.
\]
Here, $k/N$ encodes only the grid index and does not represent the physical diffusion time $t_k$.
Although $\sigma_{t_k}^G(\vy,\vx)$ is defined for all $\vx,\vy\in[K]^d$, only values with $\dHam(\vx,\vy)=1$ enter the score-estimation objective and reverse dynamics, since $\mQ(\vy,\vx)=0$ otherwise.
Thus, the network needs only represent the score at the prescribed grid times and on pairs of states at Hamming distance one.
Let $\mathcal F(L,W,S,B;M):=\{\sigma^G:G\in\mathcal G(L,W,S,B)\}$ denote the resulting clipped score class.
We assume that the learned score function $\hat{\sigma}$ is an $\varepsilon_{\mathrm{opt}}$-approximate  minimizer of the empirical objective \eqref{trobj}, i.e., 
\begin{equation}
\label{eq:score_estimate}
\widehat{\mathcal L}_{\mathcal T}^{\mathrm{DSE}}(\widehat\sigma)
\le
\inf_{\sigma\in\mathcal F(L,W,S,B;M)}
\widehat{\mathcal L}_{\mathcal T}^{\mathrm{DSE}}(\sigma)
+\varepsilon_{\mathrm{opt}}, 
\end{equation}
where $\varepsilon_{\mathrm{opt}}\ge0$ is the optimization error.

\paragraph{Reverse sampler.}
At forward time $t$, the learned score $\widehat\sigma$ provides the estimated reverse rates $\mQ(\vy,\vx)\widehat\sigma_t(\vy,\vx)$ for $\vx\neq\vy$.
Starting from the stationary distribution $\pi$ at time $T$, a reverse sampler uses these rate estimates to approximate the reverse-time evolution down to the early-stopping time $\delta$.
Examples of fixed-grid samplers include $\tau$-leaping \citep{gillespie_approximate_2001,campbell_continuous_2022}, as well as the Euler method and the Tweedie $\tau$-leaping sampler \citep{lou_discrete_2024}.
Non-fixed-grid methods include Gillespie-based samplers \citep{gillespie_general_1976,liu_think_2025} and uniformization \citep{jensen_markoff_1953,chen_uniformization_2024}, with score evaluations at random or adaptively chosen times.
We focus on fixed-grid samplers, whose score evaluations are restricted to a prescribed deterministic time grid $\mathcal T$.

\begin{definition}[Fixed-grid score-based sampler]
A fixed-grid score-based sampler $\mathcal A$ takes as input an initial draw $X_T\sim\pi$, a grid $\mathcal T$, and the grid score functions $\{\sigma_{t_k}\}_{k=1}^N$, and returns a random state $X_\delta^{\mathcal A,\mathcal T,\sigma}\in[K]^d$.
We denote its output distribution by $Q_\delta^{\mathcal A,\mathcal T,\sigma}$.
\end{definition}

For the fixed grid $\mathcal T$ and the learned score $\widehat\sigma$ in \eqref{eq:score_estimate}, we abbreviate the output state and distribution as $X_\delta^{\mathcal A}$ and $Q_\delta^{\mathcal A}$, respectively.
For $M\ge1$, define the class of $M$-bounded grid scores by
\(\Sigma_M(\mathcal T) := \left\{ \sigma: M^{-1}\le \sigma_{t_k}(\vy,\vx)\le M \text{ for every } k\in[N],\ \vx,\vy\in[K]^d \right\}.
\)
We formulate our sampler-agnostic analysis through the following assumption, which separates the score-estimation error from an approximation remainder accounting for initialization and time discretization.

\begin{assumption}[Score-based sampler guarantee]
\label{ass:score_based_sampler}
For any $P_0\in\Delta_{K^d}$, grid $\mathcal T$, $M\ge1$, and score $\sigma\in\Sigma_M(\mathcal T)$,
\(\KL\left(
P_\delta\Vert Q_\delta^{\mathcal A,\mathcal T,\sigma}
\right)
\le
C_{\mathcal A}
\left\{
\mathcal L_{\mathcal T}^{\mathrm{SE}}(\sigma;P_0)
+
\mathcal R_{\mathcal A}(\mathcal T;M)
\right\},\)
where $C_{\mathcal A}>0$ depends only on the sampler.
The remainder $\mathcal R_{\mathcal A}(\mathcal T;M)\ge0$ is deterministic and independent of $P_0$ and of the particular score values beyond the bound $M$.
\end{assumption}

Guarantees based on interval-weighted grid score-entropy loss are established for fixed-grid samplers under uniform discrete diffusion in \citet{liang_discrete_2025,dmitriev_efficient_2026}.
In particular, Theorem~1 of \citet{dmitriev_efficient_2026}, restated in our notation as Theorem~\ref{thm:tau_leaping_sampler}, provides a guarantee for $\tau$-leaping under its stated grid conditions, with remainder $\mathcal R_\tau(\mathcal T;M) = e^{-T}d\log K + \Delta d\log(K/\Delta)$, where $\Delta$ denotes the largest grid interval.

\section{Statistical Guarantees}
\subsection{Guarantee in Total Variation}

We begin our study of distribution estimation under the TV loss, with proofs appearing in Appendix~\ref{sec:tv_bounds}.
For any fixed-grid sampler satisfying Assumption~\ref{ass:score_based_sampler}, Theorem~\ref{thm:tv} shows that uniform discrete diffusion attains the rate $\mathcal{O}(\sqrt{\mathfrak{s}_n(P_0)/n})$ through score estimation and reverse-time sampling.
We complement this upper bound with a matching minimax lower bound over $\mathcal P_{s_n}$ in Theorem~\ref{thm:tv_lower}.
Let $D_n:=|\{X_0^{(1)},\ldots,X_0^{(n)}\}|$ denote the number of distinct observations in the training sample.
\begin{theorem}[TV upper bound for uniform discrete diffusion]
\label{thm:tv}
Define the natural clipping bound
\(M:=\frac{1+(K-1)e^{-t_1}}{1-e^{-t_1}},\)
which uniformly bounds the true score between states of Hamming distance one for all $t\ge t_1$.
Suppose that the fixed-grid sampler $\mathcal A$ satisfies Assumption~\ref{ass:score_based_sampler} and that
\(\mathcal R_{\mathcal A}(\mathcal T;M)
  \lesssim\mathfrak{s}_n(P_0)/n,
  \; \varepsilon_{\mathrm{opt}}
  \lesssim \mathfrak{s}_n(P_0)/n, \;  d\delta
  \lesssim\sqrt{\mathfrak{s}_n(P_0)/n}, \text{ and } T\ge1.\)
Then there exists a clipped ReLU score class $\mathcal F(L,W,S,B;M)$ with
\(L\lesssim d^2 \log^2\left(\frac{K}{1-e^{-t_1}}\right)
  +\log^2\left(\frac{nT}{\mathfrak{s}_n(P_0)}\right),\)
\(W\lesssim dKN+dD_n +d^3\log^3\left(\frac{K}{1-e^{-t_1}}\right)
  +\log^3\left(\frac{nT}{\mathfrak{s}_n(P_0)}\right),\)
\(S\lesssim dK^2N +d^2D_n\log\left(\frac{K}{1-e^{-t_1}}\right)
  +dD_n\log\left(\frac{nT}{\mathfrak{s}_n(P_0)}\right)
  +d^4\log^4\left(\frac{K}{1-e^{-t_1}}\right)
  +\log^4\left(\frac{nT}{\mathfrak{s}_n(P_0)}\right),\)
\(B\lesssim N +\left(\frac{K}{1-e^{-t_1}}\right)^{2d}
  +\frac{dnMT}{\mathfrak{s}_n(P_0)},\)
such that every $\varepsilon_{\mathrm{opt}}$-optimizer $\widehat{\sigma}$ of $\widehat{\mathcal L}_{\mathcal T}^{\mathrm{DSE}}$ over this class satisfies 
\[
\E_{P_0}\TV\left(P_0,Q_\delta^{\mathcal A}\right)
  \lesssim \sqrt{\frac{\mathfrak{s}_n(P_0)}{n}}.
\]
\end{theorem}

Theorem~\ref{thm:tv} is stated for a generic fixed-grid sampler and applies once its remainder term $\mathcal R_{\mathcal A}$ is controlled at the stated rate. For the $\tau$-leaping sampler, instantiating Theorem~\ref{thm:tv} with this remainder yields an explicit sampling grid and matching architecture bounds as shown in Corollary~\ref{cor:tv}.
\begin{corollary}[TV guarantee for $\tau$-leaping sampler]
\label{cor:tv}
For the $\tau$-leaping sampler, the grid choice
\(\mathcal T=\left\{\delta=0,\Delta,\ldots,(N-1)\Delta,T=\log\left(
  \frac{edn\log K}{\mathfrak{s}_n(P_0)}\right)\right\},\)
where
$\Delta=\frac{\mathfrak{s}_n(P_0)}{dn\log(Kdn/\mathfrak{s}_n(P_0))},
  \text{ and } N=\left\lceil
  \frac{T}{\Delta}\right\rceil$
satisfies the sampler conditions of Theorem~\ref{thm:tv}.
The clipped ReLU score class can be chosen with
\(M=\frac{1+(K-1)e^{-\Delta}}{1-e^{-\Delta}}, \; L\lesssim d^2\log^2\left(
  \frac{Kdn}{\mathfrak{s}_n(P_0)}\right),\) \(W\lesssim \frac{d^2Kn}{\mathfrak{s}_n(P_0)}
  \log^2\left(\frac{Kdn}{\mathfrak{s}_n(P_0)}\right) + dD_n
  +d^3\log^3\left(\frac{Kdn}{\mathfrak{s}_n(P_0)}\right),\)
\(S\lesssim \frac{d^2K^2n}{\mathfrak{s}_n(P_0)}
  \log^2\left(\frac{Kdn}{\mathfrak{s}_n(P_0)}\right)
  +d^2D_n\log\left(\frac{Kdn}{\mathfrak{s}_n(P_0)}\right)
  +d^4\log^4\left(\frac{Kdn}{\mathfrak{s}_n(P_0)}\right),
\) and
\(B\lesssim \left[ \frac{Kdn}{\mathfrak{s}_n(P_0)} \log\left(
  \frac{Kdn}{\mathfrak{s}_n(P_0)}\right) \right]^{2d}.
\)
\end{corollary}

\paragraph{Interpretation of the sampling and network parameters.}
The choices of the terminal time \(T\) and grid spacing \(\Delta\) in Corollary~\ref{cor:tv} calibrate the initialization and discretization errors to the statistical scale \(\mathfrak s_n(P_0)/n\), while no early stopping is needed under TV loss, so we take \(\delta=0\).
The resulting number of score evaluations and clipping level satisfy
\(N=\widetilde{\mathcal O}(dn/\mathfrak s_n(P_0))\) and
\(M=\widetilde{\mathcal O}(Kdn/\mathfrak s_n(P_0))\), respectively.
Together with the displayed bounds on the network depth \(L\), width \(W\), and sparsity \(S\), these quantities depend polynomially on \(K\), \(d\), and \(n/\mathfrak s_n(P_0)\), rather than on the ambient cardinality \(K^d\).
Thus, the TV guarantee avoids polynomial dependence on the ambient state-space size in the number of score evaluations, the clipping level, and the network size.
The number of distinct states in the sample, $D_n$, can be upper bounded by $\mathfrak s_n(P_0)$ in expectation and by high probability by Lemma~\ref{lem:D_n_s_n}.

\paragraph{Minimax lower bound.}
As a classical benchmark, Lemma~\ref{lem:tv_emp} shows that the empirical distribution $\widehat P_0$ also achieves the rate $\sqrt{\mathfrak{s}_n(P_0)/n}$.
Theorem~\ref{thm:tv_lower} shows that this rate is minimax optimal over $\mathcal P_{s_n}$.

\begin{theorem}[Minimax rate under TV loss]
\label{thm:tv_lower}
For every $2\le s_n\le \min\{n,K^d\}$,
\[\inf_{\widetilde P}\sup_{P_0\in\mathcal P_{s_n}} \E_{P_0}\TV(P_0,\widetilde P)
  \asymp \sqrt{s_n/n}.
\]
\end{theorem}

\paragraph{Comparison with existing TV bounds.}
Since the score-entropy loss is naturally aligned with KL divergence, relatively few works derive TV guarantees for uniform discrete diffusion.
To our knowledge, only \cite{zhang_generalization_2026} obtain a TV bound for uniform discrete diffusion, of order
\(\widetilde{\mathcal O}\!\left(\sqrt{dK^{d+1}/n}\right)\), which becomes vacuous when $n \ll K^d$.
In contrast, our rate depends on the effective support size \(\mathfrak s_n(P_0)\) and, by Theorem~\ref{thm:tv_lower}, attains the minimax-optimal rate over the corresponding effective support classes.
In the non-saturated regimes of \eqref{eq:s_n_behavior}, this yields rates
\(\mathcal O(n^{(\alpha^{-1}-1)/2})\) and
\(\mathcal O(\sqrt{\log n/n})\)
for Zipfian and geometric decay, respectively.

\subsection{Guarantee in KL Divergence}
We now turn to distribution estimation under KL loss, with proofs appearing in Appendix~\ref{sec:kl_bounds}.
Unlike in the TV setting, the empirical distribution $\widehat P_0$ generally has infinite expected KL risk.
Thus, obtaining a meaningful KL guarantee requires assigning positive probability beyond the observed training samples, making KL loss a stringent measure of distributional generalization.
For any fixed-grid sampler satisfying Assumption~\ref{ass:score_based_sampler}, Theorem~\ref{thm:kl} shows that uniform discrete diffusion with early-stopping time $\delta=1/n$ attains the rate $\mathcal{O}\!\left(\frac 1n \mathfrak{s}_n(P_0)\log(eK^d/\mathfrak{s}_n(P_0))\log n\right)$.
We complement this upper bound with a minimax lower bound over $\mathcal P_{s_n}$ in Theorem~\ref{thm:kl_lower}, showing that uniform discrete diffusion is minimax optimal up to a factor of $\log n$.
Together, these results show that uniform discrete diffusion can generalize well beyond the observed training samples rather than merely memorizing the empirical distribution.

\begin{theorem}[KL guarantee for uniform discrete diffusion]
\label{thm:kl}
Set early-stopping time $\delta:=1/n$.
Define the natural clipping bound
\(M:=\frac{1+(K-1)e^{-t_1}}{1-e^{-t_1}},\)
which uniformly bounds the true score between states of Hamming distance one for all $t\ge t_1$.
Suppose that the fixed-grid sampler $\mathcal A$ satisfies Assumption~\ref{ass:score_based_sampler} and that
\(\mathcal R_{\mathcal A}(\mathcal T;M) \lesssim \left(
  \frac{1-e^{-1/n}}{K}\right)^d, \; \varepsilon_{\mathrm{opt}} \lesssim
  \left(\frac{1-e^{-1/n}}{K}\right)^d, \text{ and } T\ge1.\)
Then there exists a clipped ReLU score class $\mathcal F(L,W,S,B;M)$ with
\(L\lesssim d^2\log^2(Kn)+\log^2 T,\) \(W\lesssim dKN+dD_n +d^3\log^3(Kn)+\log^3 T,\) 
\(S\lesssim dK^2N +d^2D_n\log(Kn) +dD_n\log T +d^4\log^4(Kn) +\log^4 T,\) and \(B\lesssim N +\left(\frac{K}{1-e^{-t_1}}\right)^{2d} +dMT\left(\frac{K}{1-e^{-1/n}}\right)^d, \)
such that every $\varepsilon_{\mathrm{opt}}$-optimizer $\widehat{\sigma}$ of $\widehat{\mathcal L}_{\mathcal T}^{\mathrm{DSE}}$ over this class satisfies 
\[
\E_{P_0}\KL\left(P_0\Vert Q_\delta^{\mathcal A}\right) \lesssim
  \frac{1}{n}\mathfrak{s}_n(P_0)
  \log\!\left(eK^d/\mathfrak{s}_n(P_0)\right)\log n. 
\]
\end{theorem}
Once again, for the $\tau$-leaping sampler, we can obtain specific sampling grid and network architecture choices as highlighted in Corollary~\ref{cor:kl}.
\begin{corollary}[KL guarantee for $\tau$-leaping sampler]
\label{cor:kl}
For the $\tau$-leaping sampler, the grid $\mathcal T=\{\delta=\tfrac1n,\tfrac1n+\Delta,\ldots,\tfrac1n+(N-1)\Delta,T\}$ with $T=\log\!\big[ed(K/(1-e^{-1/n}))^d\log K\big]$, $\Delta=\big(d^2(K/(1-e^{-1/n}))^d\log(Kn)\big)^{-1}$, and $N=\lceil(T-1/n)/\Delta\rceil$ satisfies the sampler conditions of Theorem~\ref{thm:kl}. The clipped ReLU score class can be chosen with $M=\frac{1+(K-1)e^{-(1/n+\Delta)}}{1-e^{-(1/n+\Delta)}}$, $L\lesssim d^2\log^2(Kn)$, $W\lesssim d^4K(K/(1-e^{-1/n}))^d\log^2(Kn)$, $S\lesssim d^4K^2(K/(1-e^{-1/n}))^d\log^2(Kn)$, and $B\lesssim(K/(1-e^{-1/n}))^{2d}$.
\end{corollary}

\paragraph{Interpretation of the sampling and network parameters.}
Unlike under TV loss, early stopping is essential under KL loss.
We take \(\delta=1/n\) since the early-stopped empirical distribution \(\widehat P_{1/n}\) already achieves the rate of Theorem~\ref{thm:kl}.
The choices of the terminal time \(T\) and grid spacing \(\Delta\) in Corollary~\ref{cor:kl} calibrate the initialization and discretization errors to the sampler-accuracy scale \(\left((1-e^{-1/n})/K\right)^d\), which is much smaller than the final statistical KL rate, unlike in the TV case.
This difference arises from the absence of a triangle inequality for KL divergence.
Instead, we use Lemma~\ref{lem:local-kl-transfer}, which requires the sampling error to be controlled at the uniform lower bound on the probabilities at time \(1/n\).
Although the clipping level \(M=\widetilde {\mathcal O}(Kn)\) and network depth \(L\) remain mild, the number of sampling steps \(N\), network width \(W\), and sparsity \(S\) incur exponential dependence on \(d\).
This unfavorable computational dependence is not intrinsic to the statistical KL rate: in Appendix~\ref{app:exact_sampler}, we construct an exact one-step empirical reverse sampler that achieves the same statistical rate with only \(\mathcal O(d(D_n+K))\) operations per generated sample.

\paragraph{Minimax lower bound.}
The following theorem characterizes the minimax rate over the effective support class $\mathcal P_{s_n}$, showing that uniform discrete diffusion achieves the minimax rate up to a factor of $\log n$.

\begin{theorem}[Minimax rate under KL loss]
\label{thm:kl_lower}
For every $2\le s_n \le \min\{n,K^d\}$,
\[
\inf_{\widetilde P}\sup_{P_0\in\mathcal P_{s_n}} \E_{P_0}\KL(P_0\Vert\widetilde P)
  \asymp \frac{s_n }{n} \log(eK^d/s_n).
\]
\end{theorem}

\paragraph{Comparison with classical distribution estimators.}

There has been substantial interest in estimating discrete distributions
\citep{braess_achieve_2002,paninski_variational_2004,braess_bernstein_2004,van_nearly_2025}
and the related missing mass problem
\citep{good_population_1953,orlitsky_competitive_2015}.
Classical asymptotic theory shows that, for a fixed state space of size \(K^d\), the minimax expected KL risk is \((K^d-1)/(2n)+o(1/n)\) as \(n\to\infty\).
This fixed state space asymptotic does not describe the large state space regime relevant to discrete diffusion.

Since the empirical distribution \(\widehat P_0\) generally has infinite KL risk, classical discrete distribution estimation has long relied on smoothing to assign positive probability to unseen states.
Among fixed additive-smoothing estimators of the form \((n\widehat P_0(x)+\lambda)/(n+\lambda K^d)\), the Laplace estimator with \(\lambda=1\) achieves the minimax-optimal expected risk of order \(\Theta \left(\log(1+K^d/n)\right)\) over the unrestricted class \(\Delta_{K^d}\)
\citep{paninski_variational_2004,mourtada_estimation_2026}.
However, this rate does not vanish when \(n\ll K^d\).
The adaptive additive-smoothing estimator with \(\lambda=D_n/K^d\)
\citep{mourtada_estimation_2026}
and the absolute-discounting estimator
\citep{falahatgar_power_2017}
both achieve the minimax-optimal rate over the effective support class of Theorem~\ref{thm:kl_lower}.
Moreover, when the effective support size is maximal, i.e., $s_n = \min\{n, K^d\}$, which corresponds to the unrestricted class, the minimax rates in Theorems~\ref{thm:tv_lower} and~\ref{thm:kl_lower} recover the corresponding classical unrestricted minimax rates (Lemma~\ref{lem:max-effective-support-kl}).
Thus, minimax theory based on effective support size extends the classical unrestricted theory while remaining informative for distributions whose effective support size is small.

\paragraph{Comparison with existing KL bounds.}
Several recent works derive finite-sample KL guarantees for uniform discrete diffusion.
\cite{cho_minimax_2026} establish KL rate of order $\widetilde{\mathcal O}(K^d/n)$, which they claim is nearly minimax optimal.
However, their guarantees are established in the regime $n\gtrsim K^d$, which does not capture the regime $n\ll K^d$ typical of practical discrete diffusion.
\cite{wakasugi_state_2025} obtain a high-probability KL bound of the same order, together with a refined bound that removes the explicit polynomial dependence on $K^d$.
This refinement requires decay of the target distribution's coefficients in an eigenbasis of the generator, a structural condition that is difficult to verify in practice.
Moreover, the refined bound is not minimax optimal.
\cite{srikanth_sample_2026} derive a high-probability KL guarantee with explicit neural-network complexity, but require unrealistic full support assumption and incur polynomial dependence on the maximum initial score ratio, which may be as large as the ambient cardinality $K^d$.
In contrast, our bound adapts to the effective support size \(\mathfrak s_n(P_0)\) without such structural assumptions, are nearly minimax optimal under effective support class $\mathcal P_{s_n}$, and yields rates \(\mathcal O\!\left(n^{\alpha^{-1}-1}\log\!\frac{eK^d}{n^{1/\alpha}}\log n\right)\)
and
\(\mathcal O\!\left(\frac{\log^2 n}{n}\log\!\frac{eK^d}{\log n}\right)\)
for Zipfian and geometric decay, respectively, in the non-saturated regimes of \eqref{eq:s_n_behavior}.

\section{Experiments}
\label{sec:experiments}

\paragraph{Controlled populations.}
We conduct a controlled experiment to evaluate whether uniform discrete diffusion models learn populations with smaller effective support sizes more accurately when the size of the ambient state space is held fixed. We first train a byte-level BPE tokenizer with vocabulary size $K=2048$ on TinyStories \citep{eldan_tinystories_2023}. Using the resulting tokenized corpus, we train an 88M-parameter, 12-layer decoder-only Transformer as an autoregressive teacher model that predicts each token conditioned on the preceding tokens. We then use the frozen teacher at temperatures $\{0.4,0.7,1.0\}$ to define three population distributions over $d=32$-token sequences, all on the same ambient state space of size $K^d=2048^{32}$. Lower temperatures concentrate the teacher's next-token probabilities and therefore produce populations with smaller effective support sizes. 
\paragraph{Training and Evaluation.}
For each population, we consider seven training-set sizes ($n=10\mathrm{K}$--$1\mathrm{M}$).
For every population and training-set size, we train a 96M-parameter uniform SEDD model \citep{lou_discrete_2024} for 50 epochs.
Each experimental setting is repeated three times. Directly computing full-sequence TV or KL is infeasible because it would require tractable access to the model probabilities over $2048^{32}$ possible sequences.
We therefore draw 100K samples from each trained SEDD model and 100K fresh samples from the corresponding teacher population and compare them using two sample-based discrepancies.
Teacher-NLL KS compares the distributions of sequence negative log-likelihood under the teacher; the corresponding population discrepancy lower-bounds, but does not estimate, full-sequence TV.
The second metric, $1-\mathrm{MAUVE}$, compares the samples in a fixed GPT-2-large representation \citep{pillutla_mauve_2021}.
We apply both metrics to two independent sets of 100K samples from each teacher population as a teacher-versus-teacher null control, which indicates their finite-sample evaluation floors. Lower values indicate better agreement for both metrics. 
\paragraph{Results.}
Figure~\ref{fig:fixed-pass-results} shows that both teacher-NLL KS and $1-\mathrm{MAUVE}$ decrease as the training-set size increases.
At every training-set size, uniform SEDD models trained on lower-temperature populations achieve smaller discrepancies under both metrics.
Since these populations have smaller effective support sizes while sharing the same ambient state space, the results support effective support size as a meaningful measure of statistical difficulty for uniform discrete diffusion.
\begin{figure}[t]
    \centering
    \includegraphics[width=\linewidth]{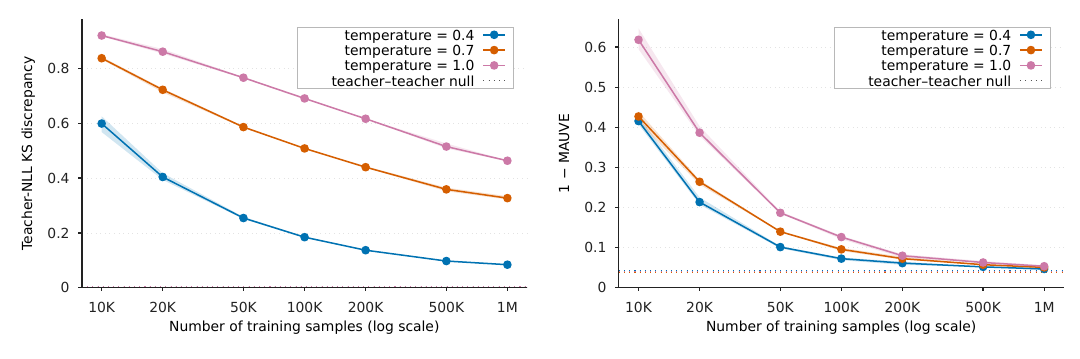}
    \caption{%
    \textbf{Uniform discrete diffusion on controlled populations.}
    Teacher-NLL KS (left) and $1-\mathrm{MAUVE}$ (right) versus
    training-set size for teacher temperatures $\{0.4,0.7,1.0\}$,
    with ambient state space $2048^{32}$.
    Lines and markers show means over three runs after 50 training
    epochs, and shaded bands span the minimum and maximum across runs.
    Dotted lines indicate teacher--teacher null controls.
    Lower values indicate better agreement between population and sample distributions. 
    }
    \label{fig:fixed-pass-results}
\end{figure}

\section{Conclusion}
\label{sec:conclusion}

In this paper, we develop a statistical theory for uniform discrete diffusion models on finite product spaces whose ambient cardinality $K^d$ can be astronomically large.
Our analysis is governed by the effective support size $\mathfrak s_n(P_0)$, a sample-size-dependent measure of distributional complexity.
We show that uniform discrete diffusion achieves the minimax-optimal TV rate $\mathcal{O}(\sqrt{\mathfrak s_n(P_0)/n})$.
Under KL loss, early stopping at time $1/n$ yields the rate $\mathcal{O}\left(\frac{1}{n}\mathfrak{s}_n(P_0)\log(eK^d/\mathfrak{s}_n(P_0))\log n\right)$, which is within a factor of $\log n$ of the minimax benchmark.
This KL guarantee captures generalization beyond the observed samples, since the empirical distribution generally assigns zero probability to unseen states and therefore has infinite KL risk.
These minimax rates identify three regimes in which consistent distribution estimation is impossible, possible only under TV loss, or possible under both TV and KL losses.
Our controlled experiment with uniform SEDD models complements these theoretical findings and supports effective support size as a meaningful measure of statistical difficulty beyond ambient cardinality alone.
Important directions for future work include closing the remaining logarithmic gap under KL loss, controlling optimization error in score estimation, and extending the analysis to masking discrete diffusion models.

% \section*{Disclosure of AI Use}

% We used generative AI tools to assist with writing and editing experimental code, checking mathematical proofs developed by the authors, retrieving potentially relevant related work, and polishing the manuscript for clarity and readability.
% The research ideas and original mathematical proofs were developed by the authors without generative AI assistance.
% All AI-assisted outputs were thoroughly reviewed and verified by the authors.
% We take full responsibility for the final content of this paper, including its text, mathematical claims and proofs, experimental results, citations, and accompanying code.
% \clearpage
% \singlespacing
% \bibliographystyle{plainnat}
% \bibliography{references}

\clearpage
\appendix
\numberwithin{equation}{section} \numberwithin{theorem}{section}

% Included by main_arxiv.tex; compile the main file.
\section*{Appendices}
\phantomsection
\addcontentsline{toc}{section}{Appendices}
\startcontents[sections]
\printcontents[sections]{l}{1}{\setcounter{tocdepth}{2}}
\paragraph{Notation.}
We retain the notation of the main text and introduce only appendix-specific shorthand here.
For any probability distribution $\mu$ on $[K]^d$, write $\mu_t:=\mu \mP_t$ for its marginal after evolving under the forward process for time $t$.
When the data distribution is $P_0$, we abbreviate its point masses as $p_{\vx}:=P_0(\vx)$.
For $t>0$, define
\[
B_0(t):=\left(\frac{K}{1-e^{-t}}\right)^d, \qquad
  B_1(t):=\frac{1+(K-1)e^{-t}}{1-e^{-t}}.
\]
For $u\in\mathbb R$, write $[u]_+:=\max\{u,0\}$; in the neural-network constructions, $\rho(u):=\max\{u,0\}$ denotes the ReLU activation.

\section{Preliminary Results}

\subsection{Some Facts on Effective Support Size}

\begin{lemma}[Effective support size under representative decay regimes]
\label{lem:s_n_behavior}
Let \(p_{(1)}\ge\cdots\ge p_{(|\supp(P_0)|)}>0\) denote nonzero masses of $P_0$ in decreasing order.
Suppose \(n\ge2\).
In each condition below, \(\asymp\) means that the comparison holds uniformly over \(1\le j\le|\supp(P_0)|\) with universal constants.
Then
\[
\mathfrak s_n(P_0)\asymp
\begin{cases}
\min\{|\supp(P_0)|,n\}, & \text{if } p_{(j)}\asymp |\supp(P_0)|^{-1},\\[2pt]
\min\{|\supp(P_0)|,n^{1/\alpha}\}, & \text{if } p_{(j)}\asymp j^{-\alpha},\\[2pt]
\min\{|\supp(P_0)|,\log n\}, & \text{if } p_{(j)}\asymp e^{-cj},
\end{cases}
\]
where \(\alpha>1\) and \(c>0\) are fixed.
The implicit constants in the conclusion depend only on the comparison constants and, in the second and third cases, on \(\alpha\) and \(c\), respectively.
\end{lemma}

\begin{proof}
Since \(u\asymp v\) implies \(\min\{nu,1\}\asymp\min\{nv,1\}\), we may work directly with the representative sequences in the three cases.

If \(p_{(j)}\asymp |\supp(P_0)|^{-1}\), direct calculation gives \(\mathfrak s_n(P_0)\asymp \min\{|\supp(P_0)|,n\}\).

If \(p_{(j)}\asymp j^{-\alpha}\), then
\[
\mathfrak s_n(P_0)\asymp\sum_{j=1}^{|\supp(P_0)|}\min\{nj^{-\alpha},1\}.
\]
If \(|\supp(P_0)|\le n^{1/\alpha}\), all summands equal \(1\), so \(\mathfrak s_n(P_0)\asymp|\supp(P_0)|\).
If \(|\supp(P_0)|>n^{1/\alpha}\), the first \(\lfloor n^{1/\alpha}\rfloor\) summands equal \(1\), while
\[
n\sum_{j>\lfloor n^{1/\alpha}\rfloor}j^{-\alpha}\lesssim n^{1/\alpha}.
\]
Thus \(\mathfrak s_n(P_0)\asymp n^{1/\alpha}\), and hence \(\mathfrak s_n(P_0)\asymp\min\{|\supp(P_0)|,n^{1/\alpha}\}\).

If \(p_{(j)}\asymp e^{-cj}\), then
\[
\mathfrak s_n(P_0)\asymp\sum_{j=1}^{|\supp(P_0)|}\min\{ne^{-cj},1\}.
\]
If \(|\supp(P_0)|\le c^{-1}\log n\), all summands equal \(1\), so \(\mathfrak s_n(P_0)\asymp|\supp(P_0)|\).
If \(|\supp(P_0)|>c^{-1}\log n\), the summands equal \(1\) for \(j\le\lfloor c^{-1}\log n\rfloor\), while
\[
n\sum_{j>\lfloor c^{-1}\log n\rfloor}e^{-cj}\lesssim 1.
\]
Thus \(\mathfrak s_n(P_0)\asymp\log n\), and hence \(\mathfrak s_n(P_0)\asymp\min\{|\supp(P_0)|,\log n\}\).
\end{proof}

\begin{lemma}[Effective support size and distinct observations]
\label{lem:D_n_s_n}
Let \(X_1,\ldots,X_n\) be i.i.d. samples from \(P_0\), and let \(D_n:=|\{X_1,\ldots,X_n\}|\) denote the number of distinct observations.
Then
\[
(1-e^{-1})\mathfrak s_n(P_0)\le \E_{P_0}D_n\le \mathfrak s_n(P_0).
\]
Moreover, for every \(\eta\in(0,1)\), with probability at least \(1-\eta\),
\[
D_n\le2\bigl(\mathfrak s_n(P_0)+\log\eta^{-1}\bigr).
\]
\end{lemma}

\begin{proof}
We have \(\E_{P_0}D_n=\sum_x\{1-(1-P_0(x))^n\}\).
Since \(1-(1-u)^n\le\min\{nu,1\}\) for \(u\in[0,1]\), it follows that \(\E_{P_0}D_n\le\mathfrak s_n(P_0)\).
Conversely, \(1-(1-u)^n\ge1-e^{-nu}\ge(1-e^{-1})\min\{nu,1\}\), so \(\E_{P_0}D_n\ge(1-e^{-1})\mathfrak s_n(P_0)\).
Finally, \citet[Lemma~10.7]{mourtada_estimation_2026} gives, with probability at least \(1-\eta\),
\[
D_n\le2\bigl(\E_{P_0}D_n+\log\eta^{-1}\bigr),
\]
and the claim follows from \(\E_{P_0}D_n\le\mathfrak s_n(P_0)\).
\end{proof}

\begin{lemma}[Recovery of the unrestricted KL rate]
\label{lem:max-effective-support-kl}
Let $s_{\max}:=\min\{K^d,n\}$.
Then,
\[
\log\left(1+\frac{K^d}{n}\right)
\le
\frac{s_{\max}}{n}\log\left(\frac{eK^d}{s_{\max}}\right)
\le
2\log\left(1+\frac{K^d}{n}\right).
\]
\end{lemma}

\begin{proof}
If \(K^d\le n\), then \(s_{\max}=K^d\) and
\[
\frac{s_{\max}}{n}\log\left(\frac{eK^d}{s_{\max}}\right)=\frac{K^d}{n}.
\]
Since \(K^d/n\le1\), the inequalities \(x/2\le\log(1+x)\le x\) for \(x\in[0,1]\) give the result.

If \(K^d\ge n\), then \(s_{\max}=n\) and
\[
\frac{s_{\max}}{n}\log\left(\frac{eK^d}{s_{\max}}\right)
=
1+\log\left(\frac{K^d}{n}\right).
\]
Writing \(x:=K^d/n\ge1\), we have \(1+x\le ex\le(1+x)^2\), and hence
\[
\log(1+x)\le1+\log x\le2\log(1+x).
\]
Combining the two cases proves the claim.
\end{proof}

\subsection{Properties of Uniform Forward Process}

\begin{lemma}
\label{lem:upper_bounds}
For every \(t>0\) and \(\vx\in[K]^d\),
\[
P_t(\vx)\ge B_0(t)^{-1}.
\]
Moreover, for every \(t>0\) and \(\vx,\vy,\vz\in[K]^d\) with \(\dHam(\vx,\vy)=1\),
\[
B_1(t)^{-1}\le\frac{P_t(\vy)}{P_t(\vx)}\le B_1(t),
\qquad
B_1(t)^{-1}\le\frac{\mP_t(\vz,\vy)}{\mP_t(\vz,\vx)}\le B_1(t).
\]
\end{lemma}

\begin{proof}
For one coordinate, \(\mP_t^{\mathrm{coord}}(a,b)=e^{-t}\1\{a=b\}+(1-e^{-t})/K\ge(1-e^{-t})/K\).
Thus, by coordinatewise independence, \(\mP_t(\vz,\vx)\ge((1-e^{-t})/K)^d=B_0(t)^{-1}\), and averaging over \(\vz\sim P_0\) gives \(P_t(\vx)\ge B_0(t)^{-1}\).

Now suppose \(\dHam(\vx,\vy)=1\), and let \(j\) be the unique coordinate on which they differ.
All other coordinate factors cancel, so
\[
\frac{\mP_t(\vz,\vy)}{\mP_t(\vz,\vx)}
=
\frac{\mP_t^{\mathrm{coord}}(z^j,y^j)}{\mP_t^{\mathrm{coord}}(z^j,x^j)}.
\]
Since the smallest and largest entries of \(\mP_t^{\mathrm{coord}}\) are \((1-e^{-t})/K\) and \(e^{-t}+(1-e^{-t})/K\), respectively, the last ratio lies in \([B_1(t)^{-1},B_1(t)]\).

Finally,
\[
\frac{P_t(\vy)}{P_t(\vx)}
=
\sum_{\vz\in[K]^d}
\frac{P_0(\vz)\mP_t(\vz,\vx)}{P_t(\vx)}
\frac{\mP_t(\vz,\vy)}{\mP_t(\vz,\vx)}.
\]
The coefficients form a probability distribution over \(\vz\), so the same bounds hold for \(P_t(\vy)/P_t(\vx)\).
\end{proof}

\begin{lemma}[Uniform mixing of the forward process]
\label{lem:uniform-forward-mixing}
For every \(t\ge0\),
\[
\TV(P_t,\pi)\le 1-(1-e^{-t})^d\le de^{-t}.
\]
\end{lemma}

\begin{proof}
For one coordinate, \(\mP_t^{\mathrm{coord}}(a,\cdot)=e^{-t}\delta_a+(1-e^{-t})\operatorname{Unif}([K])\).
Hence, with probability \((1-e^{-t})^d\), all \(d\) coordinates are refreshed independently from the uniform distribution, regardless of the initial state.
Therefore, for some distribution \(Q_t\),
\[
P_t=(1-e^{-t})^d\pi+\bigl[1-(1-e^{-t})^d\bigr]Q_t.
\]
It follows that \(\TV(P_t,\pi)\le1-(1-e^{-t})^d\le de^{-t}\), where the last inequality follows from \(1-(1-u)^d\le du\).
\end{proof}

\begin{lemma}
\label{lem:tv_emp}
The empirical distribution \(\widehat P_0\) satisfies
\[
\E_{P_0}\!\left[\TV(P_0,\widehat P_0)\right]\le \frac{3}{2}\sqrt{\frac{\mathfrak s_n(P_0)}{n}}.
\]
\end{lemma}

\begin{proof}
Let \(N_{\vx}:=\sum_{i=1}^n\1\{X_0^{(i)}=\vx\}\) and \(H:=\{\vx\in[K]^d:nP_0(\vx)\ge1\}\).
For every \(\vx\in[K]^d\), the triangle inequality and Jensen's inequality give
\[
\E_{P_0}\left|\frac{N_{\vx}}{n}-P_0(\vx)\right|\le \min\left\{2P_0(\vx),\sqrt{\frac{P_0(\vx)}{n}}\right\}.
\]
Therefore,
\[
\E_{P_0}\!\left[\TV(P_0,\widehat P_0)\right]\le \sum_{\vx\notin H}P_0(\vx)+\frac12\sum_{\vx\in H}\sqrt{\frac{P_0(\vx)}{n}}.
\]
The first term is at most \(\mathfrak s_n(P_0)/n\le\sqrt{\mathfrak s_n(P_0)/n}\).
Since \(|H|\le\mathfrak s_n(P_0)\), the Cauchy--Schwarz inequality gives
\[
\sum_{\vx\in H}\sqrt{\frac{P_0(\vx)}{n}}\le \sqrt{\frac{|H|}{n}\sum_{\vx\in H}P_0(\vx)}\le\sqrt{\frac{\mathfrak s_n(P_0)}{n}}.
\]
Combining the two bounds proves the claim.
\end{proof}

\subsection{Auxiliary Results for KL Bounds}

Lemmas~\ref{lem:binomial-chernoff} to \ref{lem:hamming-nearest-neighbor} are used for proving Lemma~\ref{lem:kl_empirical_1/n}, which gives the KL divergence upper bound for the empirical distribution at time $1/n$, $\widehat P_{1/n}$.

\begin{lemma}[Chernoff bounds for a binomial random variable]
\label{lem:binomial-chernoff}
Let \(N\) be a binomial random variable.
For every \(\delta\in(0,1)\),
\[
\Prob\{N\le(1-\delta)\E N\}\le \exp\left(-\frac{\delta^2\E N}{2}\right),
\qquad
\Prob\{N\ge(1+\delta)\E N\}\le \exp\left(-\frac{\delta^2\E N}{2+\delta}\right).
\]
\end{lemma}

\begin{proof}
Write \(N\sim\operatorname{Bin}(M,p)\), so \(\E N=Mp\).
For every \(\lambda\in\mathbb R\), \(\E e^{\lambda N}=(1-p+pe^\lambda)^M\le\exp\{(\E N)(e^\lambda-1)\}\).
Applying Markov's inequality with \(\lambda=\log(1+\delta)\) and using \((1+\delta)\log(1+\delta)-\delta\ge\delta^2/(2+\delta)\) gives the upper-tail bound.
Applying the same argument to \(e^{-\lambda N}\) with \(\lambda=-\log(1-\delta)\) and using \(\delta+(1-\delta)\log(1-\delta)\ge\delta^2/2\) gives the lower-tail bound.
\end{proof}

\begin{lemma}
\label{lem:binomial-log-fluctuation}
Let \(N\) be a binomial random variable.
If $\E N \ge 1$,
\[
\E\left[\1\{N>0\}\log\frac{\E N}{N}\right]\lesssim \frac 1{\E N}.
\]
If \(0<\E N<1\), then \(\1\{N>0\}\log(\E N/N)\le0\).
\end{lemma}

\begin{proof}
If \(0<\E N<1\), then \(N>0\) implies \(N\ge1>\E N\), proving the second claim.
Suppose \(\lambda:=\E N\ge1\), and let \(A:=\{\lambda/2\le N\le3\lambda/2\}\) and \(Z:=(N-\lambda)/\lambda\).
On \(A\), since \(|Z|\le1/2\) and \(-\log(1+Z)\le-Z+Z^2\),
\[
\E\left[\1_A\log\frac{\lambda}{N}\right]\le-\E[\1_AZ]+\E[Z^2].
\]
Since \(\E Z=0\), \(\E Z^2\le1/\lambda\), and Lemma~\ref{lem:binomial-chernoff} gives \(\Prob(A^c)\le2e^{-\lambda/10}\), the Cauchy--Schwarz inequality yields
\[
-\E[\1_AZ]=\E[\1_{A^c}Z]\le\E[\1_{A^c}|Z|]\le\sqrt{\frac{2}{\lambda}}e^{-\lambda/20}\lesssim\frac1\lambda.
\]
Hence \(\E[\1_A\log(\lambda/N)]\lesssim1/\lambda\).

On \(A^c\), the logarithm is positive only when \(0<N<\lambda/2\), where \(\log(\lambda/N)\le\log\lambda\).
Therefore, Lemma~\ref{lem:binomial-chernoff} gives
\[
\E\left[\1_{A^c}\1\{N>0\}\log\frac{\lambda}{N}\right]\le\log\lambda\,e^{-\lambda/8}\lesssim\frac1\lambda.
\]
Combining the two bounds proves the claim.
\end{proof}

\begin{lemma}[Weighted Caro--Wei bound]
\label{lem:weighted-caro-wei}
Let \(G=(V,E)\) be a finite simple graph with positive vertex weights \((w_v)_{v\in V}\), and let \(N[v]:=\{v\}\cup\{u\in V:\{u,v\}\in E\}\) denote the closed neighborhood of \(v\).
Let
\[
\alpha(G):=\max\{|I|:I\subseteq V,\ \{u,v\}\notin E\text{ for all distinct }u,v\in I\}
\]
denote the independence number of \(G\).
Then
\[
\sum_{v\in V}\frac{w_v}{\sum_{u\in N[v]}w_u}\le\alpha(G).
\]
\end{lemma}

\begin{proof}
Let \((T_v)_{v\in V}\) be independent exponential random variables with rates \((w_v)_{v\in V}\), and define \(I:=\{v\in V:T_v=\min_{u\in N[v]}T_u\}\).
Since the \(T_v\)'s are continuous, the minima are unique almost surely.
If \(\{u,v\}\in E\) and \(u,v\in I\), then \(T_u<T_v\) and \(T_v<T_u\), a contradiction, so \(I\) is independent almost surely.
By the exponential race identity,
\[
\Prob(v\in I)=\frac{w_v}{\sum_{u\in N[v]}w_u}.
\]
Therefore,
\[
\sum_{v\in V}\frac{w_v}{\sum_{u\in N[v]}w_u}=\E|I|\le\alpha(G).
\]
\end{proof}

\begin{lemma}
\label{lem:hamming-nearest-neighbor}
Let \(X_0^{(1)},\ldots,X_0^{(n)},Y\) be i.i.d. from \(P_0\), and define
\[
R(Y):=\min_{1\le i\le n}\dHam\left(Y,X_0^{(i)}\right).
\]
Then
\[
\E_{P_0}R(Y)\lesssim \frac{\mathfrak s_n(P_0)}{n}\left\{1+\frac{\log\!\left(eK^d/\mathfrak s_n(P_0)\right)}{\log K}\right\}.
\]
\end{lemma}

\begin{proof}
For \(r=0,\ldots,d\), let \(B_r(\vy):=\{\vz\in[K]^d:\dHam(\vz,\vy)\le r\}\).
For \(r=0,\ldots,d-1\), independence gives
\[
\Prob_{P_0}\{R(Y)>r\}\le\sum_{\vy}P_0(\vy)e^{-nP_0(B_r(\vy))}.
\]
Since \(P_0(B_r(\vy))\ge P_0(\vy)\) and \(ue^{-u}\le\min\{u,1\}\),
\[
\sum_{\vy}P_0(\vy)e^{-nP_0(B_r(\vy))}\le\frac{\mathfrak s_n(P_0)}{n}.
\]
For a second bound, let \(G_r\) be the graph on \(\supp(P_0)\) in which two distinct states are adjacent if their Hamming distance is at most \(r\).
Since the closed neighborhood of \(\vy\) has total weight \(P_0(B_r(\vy))\), Lemma~\ref{lem:weighted-caro-wei} and \(ue^{-u}\le1\) give
\[
\sum_{\vy}P_0(\vy)e^{-nP_0(B_r(\vy))}\le\frac{\alpha(G_r)}{n}.
\]
If \(A\) is independent in \(G_r\), the balls \(B_{\lfloor r/2\rfloor}(\vx)\), \(\vx\in A\), are pairwise disjoint, and each contains at least \(K^{\lfloor r/2\rfloor}\) states.
Hence \(\alpha(G_r)\le K^{d-\lfloor r/2\rfloor}\), and therefore
\[
\Prob_{P_0}\{R(Y)>r\}\le\frac1n\min\left\{\mathfrak s_n(P_0),K^{d-\lfloor r/2\rfloor}\right\}.
\]
Summing the tail probabilities yields
\[
\E_{P_0}R(Y)\le\frac2n\sum_{\ell=0}^d\min\{\mathfrak s_n(P_0),K^{d-\ell}\}.
\]
There are at most \(1+\log(K^d/\mathfrak s_n(P_0))/\log K\) indices for which \(K^{d-\ell}\ge\mathfrak s_n(P_0)\), and their total contribution is at most
\[
\mathfrak s_n(P_0)\left\{1+\frac{\log(K^d/\mathfrak s_n(P_0))}{\log K}\right\}.
\]
The remaining terms form a geometric series with largest term smaller than \(\mathfrak s_n(P_0)\) and ratio \(K^{-1}\le1/2\), so their sum is at most \(2\mathfrak s_n(P_0)\).
Therefore,
\[
\sum_{\ell=0}^d\min\{\mathfrak s_n(P_0),K^{d-\ell}\}
\lesssim
\mathfrak s_n(P_0)\left\{1+\frac{\log\!\left(eK^d/\mathfrak s_n(P_0)\right)}{\log K}\right\},
\]
which proves the claim.
\end{proof}

\begin{lemma}
\label{lem:kl_empirical_1/n}
The empirical distribution at time \(1/n\), \(\widehat P_{1/n}\), satisfies
\[
\E_{P_0}\KL\left(P_0\Vert\widehat P_{1/n}\right)\lesssim \frac{\mathfrak s_n(P_0)\log\!\left(eK^d/\mathfrak s_n(P_0)\right)\log n}{n}.
\]
\end{lemma}

\begin{proof}
For \(\vy\in[K]^d\), let \(N_{\vy}:=\sum_{i=1}^n\1\{X_0^{(i)}=\vy\}\) and \(R_{\vy}:=\min_{1\le i\le n}\dHam(\vy,X_0^{(i)})\).
By coordinatewise independence, choosing an observation nearest to \(\vy\), and using the \(N_{\vy}\) observations equal to \(\vy\) when \(N_{\vy}>0\), give
\[
\widehat P_{1/n}(\vy)\ge\frac{e^{-d/n}}{n}\left(\frac{1-e^{-1/n}}{K}\right)^{R_{\vy}},
\qquad
N_{\vy}>0\ \Longrightarrow\ \widehat P_{1/n}(\vy)\ge\frac{N_{\vy}}{n}e^{-d/n}.
\]

We decompose the KL divergence according to whether \(\vy\) appears in the sample.
For the unobserved states, the first bound and \(1-e^{-1/n}\ge1/(2n)\) give
\[
\sum_{\vy:N_{\vy}=0}P_0(\vy)\log\frac{P_0(\vy)}{\widehat P_{1/n}(\vy)}
\le \sum_{\vy:N_{\vy}=0}P_0(\vy)[\log(nP_0(\vy))]_+ +\frac dn+\log(2Kn)\sum_{\vy:N_{\vy}=0}P_0(\vy)R_{\vy}.
\]
Since \(\Prob_{P_0}\{N_{\vy}=0\}=(1-P_0(\vy))^n\),
\[
\E_{P_0}\sum_{\vy:N_{\vy}=0}P_0(\vy)[\log(nP_0(\vy))]_+
\le\frac1n\sum_{\vy:nP_0(\vy)\ge1}nP_0(\vy)e^{-nP_0(\vy)}\log(nP_0(\vy))
\lesssim\frac{\mathfrak s_n(P_0)}n.
\]
Moreover, Lemma~\ref{lem:hamming-nearest-neighbor} gives
\[
\E_{P_0}\sum_{\vy}P_0(\vy)R_{\vy}\lesssim\frac{\mathfrak s_n(P_0)}n\left\{1+\frac{\log\!\left(eK^d/\mathfrak s_n(P_0)\right)}{\log K}\right\}.
\]

For the observed states, the second lower bound gives
\[
\sum_{\vy:N_{\vy}>0}P_0(\vy)\log\frac{P_0(\vy)}{\widehat P_{1/n}(\vy)}
\le\frac dn+\sum_{\vy:N_{\vy}>0}P_0(\vy)\log\frac{nP_0(\vy)}{N_{\vy}}.
\]
If \(nP_0(\vy)<1\), the corresponding summand is nonpositive.
If \(nP_0(\vy)\ge1\), Lemma~\ref{lem:binomial-log-fluctuation} gives its expectation at most \(C/n\).
Since at most \(\mathfrak s_n(P_0)\) states satisfy \(nP_0(\vy)\ge1\), the expected observed contribution is \(\lesssim(d+\mathfrak s_n(P_0))/n\).

Combining the two parts,
\[
\E_{P_0}\KL(P_0\Vert\widehat P_{1/n})\lesssim\frac{d+\mathfrak s_n(P_0)}n+\frac{\mathfrak s_n(P_0)\log(Kn)}n\left\{1+\frac{\log\!\left(eK^d/\mathfrak s_n(P_0)\right)}{\log K}\right\}.
\]
Since \(u\mapsto u\log(eK^d/u)\) is increasing on \([1,K^d]\), we have \(d\lesssim\mathfrak s_n(P_0)\log(eK^d/\mathfrak s_n(P_0))\).
Also, using \(\mathfrak s_n(P_0)\le n\) and \(K,n\ge2\),
\[
\log(Kn)\left\{1+\frac{\log\!\left(eK^d/\mathfrak s_n(P_0)\right)}{\log K}\right\}\lesssim \log\!\left(\frac{eK^d}{\mathfrak s_n(P_0)}\right)\log n.
\]
The claimed bound follows.
\end{proof}

Since KL divergence does not satisfy a triangle inequality, we use Lemma~\ref{lem:local-kl-transfer} instead.
Its proof relies on the entropy variational inequality in Lemma~\ref{lem:entropy-variational}.

\begin{lemma}[Entropy variational inequality]
\label{lem:entropy-variational}
Let \(\mu\) and \(\nu\) be probability distributions on a finite state space.
Then, for every real-valued function \(f\),
\[
\E_\mu f\le\KL(\mu\Vert\nu)+\log\E_\nu e^f.
\]
\end{lemma}

\begin{proof}
The claim is immediate if \(\KL(\mu\Vert\nu)=\infty\).
Otherwise, define \(\nu_f(\vx):=\nu(\vx)e^{f(\vx)}/\E_\nu e^f\).
Then
\[
0\le\KL(\mu\Vert\nu_f)=\KL(\mu\Vert\nu)-\E_\mu f+\log\E_\nu e^f,
\]
which proves the claim.
\end{proof}

\begin{lemma}
\label{lem:local-kl-transfer}
Let \(\mu,\nu,\xi\) be probability distributions on a finite state space.
Suppose that, for some \(B\ge1\) and \(\eta<\infty\),
\[
\nu(\vx)\ge B^{-1}\quad\text{for every }\vx,
\qquad
\KL(\nu\Vert\xi)\le\frac{\eta}{B}.
\]
Then there exists \(C_\eta<\infty\), depending only on \(\eta\), such that
\[
\KL(\mu\Vert\xi)\le2\KL(\mu\Vert\nu)+C_\eta\KL(\nu\Vert\xi).
\]
\end{lemma}

\begin{proof}
Let \(h(u):=u-1-\log u\) and \(u_{\vx}:=\xi(\vx)/\nu(\vx)\).
Since \(\nu\) has full support and \(\KL(\nu\Vert\xi)<\infty\), we also have \(\xi(\vx)>0\) for every \(\vx\).
Moreover,
\[
h(u_{\vx})=\frac{\DBr(\nu(\vx),\xi(\vx))}{\nu(\vx)}\le B\KL(\nu\Vert\xi)\le\eta.
\]
The set \(\{u>0:h(u)\le\eta\}\) is compact, and \(u\mapsto(1-u)^2/(uh(u))\) extends continuously to \(u=1\) with value \(2\).
Hence this function is bounded on the set by a constant depending only on \(\eta\), and therefore
\[
\chi^2(\nu,\xi):=\sum_{\vx}\frac{(\nu(\vx)-\xi(\vx))^2}{\xi(\vx)}
\le C_\eta\KL(\nu\Vert\xi).
\]
Applying Lemma~\ref{lem:entropy-variational} with \(f(\vx):=\log(\nu(\vx)/\xi(\vx))\) gives
\[
\KL(\mu\Vert\xi)\le2\KL(\mu\Vert\nu)+\log\!\left(1+\chi^2(\nu,\xi)\right)
\le2\KL(\mu\Vert\nu)+C_\eta\KL(\nu\Vert\xi).
\]
\end{proof}

\subsection{Score-Entropy Identities and Sampler Guarantees}

The following lemma is the discrete diffusion score-entropy loss analogue of the standard denoising identity.
\begin{lemma}
\label{lem:se_dse_equivalent}
Let \(\mu\) be a distribution on \([K]^d\), and define \(\mu_t:=\mu\mP_t\) and \(\sigma_t^{\star,\mu}(\vy,\vx):=\mu_t(\vy)/\mu_t(\vx)\).
For every \(t>0\), there exists a finite quantity \(C_t(\mu)\), independent of the candidate score \(\sigma\), such that
\[
\mathcal L_t^{\mathrm{DSE}}(\sigma;\mu)=\mathcal L_t^{\mathrm{SE}}(\sigma;\mu)+C_t(\mu).
\]
Consequently, for any positive score functions \(\sigma\) and \(\widetilde\sigma\),
\[
\mathcal L_{\mathcal T}^{\mathrm{DSE}}(\sigma;\mu)-\mathcal L_{\mathcal T}^{\mathrm{DSE}}(\widetilde\sigma;\mu)
=
\mathcal L_{\mathcal T}^{\mathrm{SE}}(\sigma;\mu)-\mathcal L_{\mathcal T}^{\mathrm{SE}}(\widetilde\sigma;\mu).
\]
\end{lemma}

\begin{proof}
Fix \(t>0\), let \(X_0\sim\mu\), and conditionally on \(X_0\), let \(X_t\sim\mP_t(X_0,\cdot)\).
For fixed \(\vx,\vy\), the conditional score satisfies
\[
\sigma_{t\mid0}^{\star}(\vy,\vx\mid X_0)=\frac{\mP_t(X_0,\vy)}{\mP_t(X_0,\vx)}.
\]
Conditioning on \(X_t=\vx\) and using Bayes' rule,
\[
\Prob(X_0=\vz\mid X_t=\vx)=\frac{\mu(\vz)\mP_t(\vz,\vx)}{\mu_t(\vx)}.
\]
Therefore,
\[
\begin{aligned}
\E\!\left[\sigma_{t\mid0}^{\star}(\vy,\vx\mid X_0)\mid X_t=\vx\right]
&=\sum_{\vz}\frac{\mP_t(\vz,\vy)}{\mP_t(\vz,\vx)}
\frac{\mu(\vz)\mP_t(\vz,\vx)}{\mu_t(\vx)}\\
&=\frac{\mu_t(\vy)}{\mu_t(\vx)}
=\sigma_t^{\star,\mu}(\vy,\vx).
\end{aligned}
\]
Conditioning the DSE loss on \(X_t\) and expanding the Bregman divergence in the SE loss therefore gives
\[
\mathcal L_t^{\mathrm{DSE}}(\sigma;\mu)=\mathcal L_t^{\mathrm{SE}}(\sigma;\mu)+C_t(\mu),
\]
where
\[
C_t(\mu):=\E_{X_t\sim\mu_t}\sum_{\vy:\mQ(\vy,X_t)>0}\mQ(\vy,X_t)\sigma_t^{\star,\mu}(\vy,X_t)\bigl(1-\log\sigma_t^{\star,\mu}(\vy,X_t)\bigr).
\]
Since \(t>0\), all forward probabilities are positive, so \(C_t(\mu)\) is finite.
Applying the identity to both \(\sigma\) and \(\widetilde\sigma\), subtracting to cancel \(C_t(\mu)\), and summing over the grid with weights \(t_k-t_{k-1}\) proves the second statement.
\end{proof}

We next record \citet[Theorem~1]{dmitriev_efficient_2026} in our notation.
As noted there, the same guarantee applies under early stopping.

\begin{theorem}[$\tau$-leaping sampler guarantee]
\label{thm:tau_leaping_sampler}
For a fixed grid \(\mathcal T=\{\delta=t_0<t_1<\cdots<t_N=T\}\), let \(\Delta:=\max_{1\le k\le N}(t_k-t_{k-1})\) and suppose \(\Delta\lesssim1\).
Then the \(\tau\)-leaping sampler initialized from \(\pi\) satisfies
\[
\KL\left(P_\delta\Vert Q_\delta^{\tau,\mathcal T,\sigma}\right)
\lesssim
\mathcal L_{\mathcal T}^{\mathrm{SE}}(\sigma;P_0)
+e^{-T}d\log K
+\Delta d\log\left(\frac{K}{\Delta}\right).
\]
\end{theorem}

\subsection{ReLU Network Constructions}

We first record two ReLU network constructions following \citet[Lemmas~F.6--F.7]{oko_diffusion_2023}.
For multiplication, we use the bounded-input construction in the proof of Lemma~F.6, restricted to \([0,1]^d\) and clipped to \([0,1]\).
We omit the proofs.

\begin{lemma}[ReLU multiplication network]
\label{lem:multiplication_network}
For every integer \(d\ge2\) and \(\epsilon\in(0,1)\), there exists a ReLU network \(\operatorname{Mult}_{d,\epsilon}:[0,1]^d\to[0,1]\) such that
\[
\sup_{u\in[0,1]^d}\left|\operatorname{Mult}_{d,\epsilon}(u)-\prod_{j=1}^d u_j\right|\le\epsilon.
\]
Moreover, the network can be chosen with
\[
L_{\mathrm{mult}}\lesssim\log d\bigl(\log\epsilon^{-1}+\log d\bigr),\qquad W_{\mathrm{mult}}\lesssim d,\qquad S_{\mathrm{mult}}\lesssim d\bigl(\log\epsilon^{-1}+\log d\bigr),\qquad B_{\mathrm{mult}}\lesssim1.
\]
\end{lemma}

\begin{lemma}[ReLU reciprocal network]
\label{lem:reciprocal_network}
For every \(\epsilon\in(0,1)\), there exists a ReLU network \(\operatorname{Rec}_{\epsilon}:\mathbb R\to\mathbb R\) such that, for every \(u\in[\epsilon,\epsilon^{-1}]\) and \(u'\in\mathbb R\),
\[
\left|\operatorname{Rec}_{\epsilon}(u')-\frac1u\right|\le\epsilon+\frac{|u'-u|}{\epsilon^2}.
\]
Moreover, the network can be chosen with
\[
L_{\mathrm{rec}}\lesssim\log^2\epsilon^{-1},\qquad W_{\mathrm{rec}}\lesssim\log^3\epsilon^{-1},\qquad S_{\mathrm{rec}}\lesssim\log^4\epsilon^{-1},\qquad B_{\mathrm{rec}}\lesssim\epsilon^{-2}.
\]
\end{lemma}

Using Lemmas~\ref{lem:multiplication_network} and~\ref{lem:reciprocal_network}, we construct a ReLU network that approximates the discrete score simultaneously over a fixed time grid.
\begin{lemma}[ReLU approximation of discrete scores on a fixed grid]
\label{lem:grid-indexed-score-network}
Let \(\mathcal T=\{\delta=t_0<t_1<\cdots<t_N=T\}\) be a fixed grid, and let \(\mu=\sum_{v=1}^D w_v\delta_{\vz_v}\) with \(w_v>0\) and \(\sum_{v=1}^D w_v=1\).
Write \(\mu_t:=\mu\mP_t\).
For every \(\varepsilon\in(0,1)\), there exists a ReLU network \(G_\varepsilon:\mathbb R^{1+2dK}\to\mathbb R\) with range contained in \([B_1(t_1)^{-1},B_1(t_1)]\) such that
\[
\max_{1\le k\le N}\max_{\dHam(\vx,\vy)=1}\left|G_\varepsilon\left(\frac{k}{N},\ve(\vy),\ve(\vx)\right)-\frac{\mu_{t_k}(\vy)}{\mu_{t_k}(\vx)}\right|\le\varepsilon.
\]
The network may be chosen with
\[
L\lesssim\log^2\left(\frac{B_0(t_1)}{\varepsilon}\right),\qquad W\lesssim d(KN+D)+\log^3\left(\frac{B_0(t_1)}{\varepsilon}\right),
\]
\[
S\lesssim dK^2N+dD\log\left(\frac{B_0(t_1)}{\varepsilon}\right)+\log^4\left(\frac{B_0(t_1)}{\varepsilon}\right),\qquad B\lesssim N+B_0(t_1)^2+\varepsilon^{-2}.
\]
\end{lemma}

\begin{proof}
We first approximate the two forward marginals, then approximate the reciprocal of the denominator, and finally multiply.
For \(k\in[N]\), define
\[
J_k(r):=\rho\left(1-N\rho\left(r-\frac{k}{N}\right)-N\rho\left(\frac{k}{N}-r\right)\right),
\]
so that \(J_k(\ell/N)=\1\{k=\ell\}\).
For \(i\in[d]\) and \(a\in[K]\), define
\(I_{k,i,a}^{\vx}(r,\ve(\vx)):=\rho(J_k(r)+(\ve(\vx))_{i,a}-1)\), and define \(I_{k,i,a}^{\vy}\) analogously.
At encoded inputs, \(I_{k,i,a}^{\vx}(\ell/N,\ve(\vx))=\1\{k=\ell,x^i=a\}\).
For \(i\in[d]\) and \(b\in[K]\), define
\[
G_{i,b}^{\vx}:=\sum_{k=1}^N\sum_{a=1}^K
\mP_{t_k}^{\mathrm{coord}}(b,a)I_{k,i,a}^{\vx},
\]
so that \(G_{i,b}^{\vx}(k/N,\ve(\vx))=\mP_{t_k}^{\mathrm{coord}}(b,x^i)\), and define \(G_{i,b}^{\vy}\) analogously.

Set \(\eta:=\min\{B_0(t_1)^{-1},\varepsilon/16\}\) and \(\varepsilon_1:=\varepsilon\eta^2/16\).
For each \(v\in[D]\), Lemma~\ref{lem:multiplication_network} applied to
\(G_{1,z_v^1}^{\vx},\ldots,G_{d,z_v^d}^{\vx}\)
gives an \(\varepsilon_1\)-approximation to \(\mP_{t_k}(\vz_v,\vx)\) at every encoded grid input.
Taking the weighted sum over \(v\) produces networks \(F_{\vx}\) and \(F_{\vy}\) satisfying
\[
|F_{\vx}-\mu_{t_k}(\vx)|\le\varepsilon_1,
\qquad
|F_{\vy}-\mu_{t_k}(\vy)|\le\varepsilon_1
\]
at every encoded input.
Clip \(F_{\vx}\) to \([B_0(t_1)^{-1},1]\), which cannot increase its error by Lemma~\ref{lem:upper_bounds}.
Since \(\eta\le\mu_{t_k}(\vx)\le1\le\eta^{-1}\), Lemma~\ref{lem:reciprocal_network}, followed by clipping to \([1,B_0(t_1)]\), gives a network \(R_{\vx}\) such that
\[
\left|R_{\vx}-\mu_{t_k}(\vx)^{-1}\right|
\le
\eta+\frac{\varepsilon_1}{\eta^2}
\le
\frac{\varepsilon}{8}.
\]
Next, clip \(F_{\vy}\) to \([0,1]\), apply
\(\operatorname{Mult}_{2,\varepsilon/(4B_0(t_1))}\)
to \(F_{\vy}\) and \(R_{\vx}/B_0(t_1)\), and multiply the output by \(B_0(t_1)\).
The resulting network \(\widetilde G_\varepsilon\) satisfies
\[
\left|
\widetilde G_\varepsilon
-
\frac{\mu_{t_k}(\vy)}{\mu_{t_k}(\vx)}
\right|
\le
\frac{\varepsilon}{4}
+
B_0(t_1)\varepsilon_1
+
\frac{\varepsilon}{8}
<
\varepsilon,
\]
where \(B_0(t_1)\varepsilon_1\le\varepsilon/16\) because \(\eta\le B_0(t_1)^{-1}\).
Finally, define
\[
G_\varepsilon
:=
\operatorname{clip}_{[B_1(t_1)^{-1},B_1(t_1)]}
(\widetilde G_\varepsilon).
\]
Since \(B_1\) is decreasing and \(t_k\ge t_1\), Lemma~\ref{lem:upper_bounds} shows that the true score lies in this interval, so the final clipping cannot increase the error.

It remains to collect the architecture bounds.
The selector and transition layers have constant depth, width \(O(dKN)\), sparsity \(O(dK^2N)\), and weights of order \(N\).
The \(2D\) multiplication networks contribute width \(O(dD)\) and sparsity \(O(dD(\log\varepsilon_1^{-1}+\log d))\), while the reciprocal contributes depth \(O(\log^2\eta^{-1})\), width \(O(\log^3\eta^{-1})\), and sparsity \(O(\log^4\eta^{-1})\).
Since \(B_0(t_1)\ge K^d\), \(\eta^{-1}\le16B_0(t_1)/\varepsilon\), and \(\log\varepsilon_1^{-1}\lesssim\log(B_0(t_1)/\varepsilon)\), these contributions give the stated bounds on \(L,W,S\).
Finally, the largest weights arise from the time selectors and reciprocal network, and \(\eta^{-2}\lesssim B_0(t_1)^2+\varepsilon^{-2}\), yielding
\[
B\lesssim N+B_0(t_1)^2+\varepsilon^{-2}.
\]
\end{proof}

\section{Proofs for TV Results}
\label{sec:tv_bounds}
\subsection{Upper Bound: Proofs of Theorem~\ref{thm:tv} and Corollary~\ref{cor:tv}}

\begin{proof}[Proof of Theorem~\ref{thm:tv}]
Fix a realization of the observations, so that \(\widehat P_0\) is supported on \(D_n\) points.
Recall that \(M=B_1(t_1)\), and set
\[
\varepsilon_{\mathrm{TV}}:=\sqrt{\frac{\mathfrak s_n(P_0)}{dnMT}}.
\]
Since \(\mathfrak s_n(P_0)\le n\), \(d\ge2\), \(M\ge1\), and \(T\ge1\), we have \(\varepsilon_{\mathrm{TV}}\in(0,1)\).
Applying Lemma~\ref{lem:grid-indexed-score-network} to \(\widehat P_0\) gives a score \(\sigma^\dagger\) such that
\[
\max_{1\le k\le N}\max_{\dHam(\vx,\vy)=1}\left|\sigma_{t_k}^\dagger(\vy,\vx)-\widehat\sigma_{t_k}^\star(\vy,\vx)\right|\le\varepsilon_{\mathrm{TV}},
\]
where \(\widehat\sigma_{t_k}^\star(\vy,\vx):=\widehat P_{t_k}(\vy)/\widehat P_{t_k}(\vx)\).

Both \(\sigma^\dagger\) and \(\widehat\sigma^\star\) take values in \([M^{-1},M]\).
Since \(f''(u)=1/u\), Taylor's theorem gives \(\DBr(u,v)\le M(u-v)^2/2\) for \(u,v\in[M^{-1},M]\).
Moreover, \(\sum_{\vy:\mQ(\vy,\vx)>0}\mQ(\vy,\vx)=d(K-1)/K\le d\), and hence
\[
\mathcal L_{\mathcal T}^{\mathrm{SE}}(\sigma^\dagger;\widehat P_0)\le\frac{dM\varepsilon_{\mathrm{TV}}^2}{2}\sum_{k=1}^N(t_k-t_{k-1})\le\frac{dMT\varepsilon_{\mathrm{TV}}^2}{2}=\frac{\mathfrak s_n(P_0)}{2n}.
\]
By Lemma~\ref{lem:se_dse_equivalent} and the \(\varepsilon_{\mathrm{opt}}\)-optimality of \(\widehat\sigma\),
\[
\mathcal L_{\mathcal T}^{\mathrm{SE}}(\widehat\sigma;\widehat P_0)\le\mathcal L_{\mathcal T}^{\mathrm{SE}}(\sigma^\dagger;\widehat P_0)+\varepsilon_{\mathrm{opt}}\lesssim\frac{\mathfrak s_n(P_0)}{n}.
\]

Applying Assumption~\ref{ass:score_based_sampler} conditionally on the observations and using the assumed sampler remainder gives
\[
\KL\left(\widehat P_\delta\Vert Q_\delta^{\mathcal A}\right)\lesssim\frac{\mathfrak s_n(P_0)}{n}.
\]
Thus, by Pinsker's inequality, \(\TV(\widehat P_\delta,Q_\delta^{\mathcal A})\lesssim\sqrt{\mathfrak s_n(P_0)/n}\).

By coupling the forward process initialized from \(\widehat P_0\) with its initial state,
\[
\TV(\widehat P_0,\widehat P_\delta)\le1-\exp\left(-\frac{d(K-1)}{K}\delta\right)\le d\delta.
\]
Therefore, the triangle inequality, Lemma~\ref{lem:tv_emp}, and the assumption on \(\delta\) yield
\[
\E_{P_0}\TV\left(P_0,Q_\delta^{\mathcal A}\right)\lesssim\sqrt{\frac{\mathfrak s_n(P_0)}{n}}+d\delta\lesssim\sqrt{\frac{\mathfrak s_n(P_0)}{n}}.
\]

It remains to verify the architecture bounds.
By Lemma~\ref{lem:grid-indexed-score-network}, it suffices to control \(B_0(t_1)/\varepsilon_{\mathrm{TV}}\).
Since \(B_0(t_1)=(K/(1-e^{-t_1}))^d\) and \(M=B_1(t_1)\le K/(1-e^{-t_1})\),
\[
\log\left(\frac{B_0(t_1)}{\varepsilon_{\mathrm{TV}}}\right)\lesssim d\log\left(\frac{K}{1-e^{-t_1}}\right)+\log\left(\frac{nT}{\mathfrak s_n(P_0)}\right),
\]
where we used \(K/(1-e^{-t_1})\ge2\), \(\mathfrak s_n(P_0)\le n\), and \(T\ge1\).
Substituting this bound into Lemma~\ref{lem:grid-indexed-score-network} gives the stated bounds on \(L,W,S\).
Finally,
\[
B_0(t_1)^2=\left(\frac{K}{1-e^{-t_1}}\right)^{2d},
\qquad
\varepsilon_{\mathrm{TV}}^{-2}=\frac{dnMT}{\mathfrak s_n(P_0)},
\]
which gives the stated bound on \(B\).
\end{proof}

\begin{proof}[Proof of Corollary~\ref{cor:tv}]
Let \(\Lambda:=\log(Kdn/\mathfrak s_n(P_0))\).
Since \(\mathfrak s_n(P_0)\le n\) and \(K,d\ge2\), we have \(\Lambda>1\).

For the stated grid, \(\Delta=\mathfrak s_n(P_0)/(dn\Lambda)\).
By Theorem~\ref{thm:tau_leaping_sampler},
\[
\mathcal R_{\tau}(\mathcal T,\theta;M)
=
e^{-T}d\log K+\Delta d\log\left(\frac{K}{\Delta}\right).
\]
The definitions of \(T\) and \(\Delta\) give
\[
e^{-T}d\log K=\frac{\mathfrak s_n(P_0)}{en},
\qquad
\Delta d\log\left(\frac{K}{\Delta}\right)
=
\frac{\mathfrak s_n(P_0)}{n}\left(1+\frac{\log\Lambda}{\Lambda}\right)
\lesssim\frac{\mathfrak s_n(P_0)}{n}.
\]
Hence \(\mathcal R_{\tau}(\mathcal T;M)\lesssim\mathfrak s_n(P_0)/n\).
Moreover, \(\delta=0\) and \(T\ge\log(2e\log2)>1\), so Theorem~\ref{thm:tv} applies.

It remains to simplify the architecture bounds.
Since \(t_1=\Delta\le1/(d\Lambda)<1\), the inequality \(1-e^{-\Delta}\ge\Delta e^{-\Delta}\) gives
\[
\frac{K}{1-e^{-\Delta}}
\le
\frac{eKdn\Lambda}{\mathfrak s_n(P_0)},
\qquad
d\Delta\le\frac1\Lambda.
\]
Therefore,
\[
\log\left(\frac{K}{1-e^{-\Delta}}\right)\lesssim\Lambda.
\]
Also, since \(\log K\le K/e\),
\[
T=\Lambda+\log\left(\frac{e\log K}{K}\right)\le\Lambda.
\]
Consequently,
\[
N=\left\lceil\frac{T}{\Delta}\right\rceil
\lesssim
\frac{dn}{\mathfrak s_n(P_0)}\Lambda^2,
\qquad
\log\left(\frac{nT}{\mathfrak s_n(P_0)}\right)\lesssim\Lambda.
\]
Substituting these bounds into Theorem~\ref{thm:tv} yields
\[
L\lesssim d^2\Lambda^2, \qquad W\lesssim \frac{d^2Kn}{\mathfrak s_n(P_0)}\Lambda^2+dD_n+d^3\Lambda^3, \qquad S\lesssim \frac{d^2K^2n}{\mathfrak s_n(P_0)}\Lambda^2+d^2D_n\Lambda+d^4\Lambda^4.
\]
Finally, \(M=B_1(\Delta)\le K/(1-e^{-\Delta})\lesssim Kdn\Lambda/\mathfrak s_n(P_0)\), and
\[
\left(\frac{K}{1-e^{-\Delta}}\right)^{2d}
\lesssim
\left(\frac{Kdn\Lambda}{\mathfrak s_n(P_0)}\right)^{2d},
\]
where we used \(d\Delta\le1/\Lambda\).
Since \(T\le\Lambda\), the remaining terms \(N\) and \(dnMT/\mathfrak s_n(P_0)\) are dominated by the same quantity.
Thus
\[
B\lesssim
\left[
\frac{Kdn}{\mathfrak s_n(P_0)}
\log\left(\frac{Kdn}{\mathfrak s_n(P_0)}\right)
\right]^{2d}.
\]
\end{proof}

\subsection{Minimax Lower Bound: Proof of Theorem~\ref{thm:tv_lower}}

\begin{proof}[Proof of Theorem~\ref{thm:tv_lower}]
The upper bound follows immediately from Lemma~\ref{lem:tv_emp} and \(\mathfrak s_n(P_0)\le s_n\) for \(P_0\in\mathcal P_{s_n}\).

For the lower bound, let
\[
m:=\left\lfloor\frac{s_n}{2}\right\rfloor,
\qquad
\gamma:=\frac{1}{2\sqrt{2}}\sqrt{\frac{m}{n}}.
\]
Since \(2\le s_n\le n\), we have \(0<\gamma\le1/4\), and since \(2m\le s_n\le K^d\), choose distinct states \(\vx_1,\vy_1,\ldots,\vx_m,\vy_m\in[K]^d\).
For each \(\theta\in\{-1,1\}^m\), define \(P_0^\theta\in\Delta_{K^d}\) by
\[
P_0^\theta(\vx_i):=\frac{1+\gamma\theta_i}{2m},
\qquad
P_0^\theta(\vy_i):=\frac{1-\gamma\theta_i}{2m},
\qquad i\in[m],
\]
and set \(P_0^\theta(\vz)=0\) elsewhere.
Since \(\mathfrak s_n(P_0^\theta)\le|\supp(P_0^\theta)|=2m\le s_n\), every \(P_0^\theta\) belongs to \(\mathcal P_{s_n}\).

Fix any possibly randomized estimator \(\widetilde P\), and define \(\widetilde\theta_i=1\) if \(\widetilde P(\vx_i)\ge\widetilde P(\vy_i)\) and \(\widetilde\theta_i=-1\) otherwise.
By the triangle inequality,
\[
\TV(\widetilde P,P_0^\theta)\ge\frac{\gamma}{2m}\sum_{i=1}^m\1\{\widetilde\theta_i\ne\theta_i\}.
\]
Let \(\Prob_\theta\) and \(\E_\theta\) denote probability and expectation under \((P_0^\theta)^{\otimes n}\), including any auxiliary randomness of \(\widetilde P\).
Averaging uniformly over \(\theta\) gives
\[
\sup_{\theta}\E_\theta\TV(\widetilde P,P_0^\theta)
\ge
\frac{\gamma}{2m}\sum_{i=1}^m\frac{1}{2^m}\sum_{\theta}\Prob_\theta(\widetilde\theta_i\ne\theta_i).
\]

Fix \(i\in[m]\) and \(\theta_{-i}\in\{-1,1\}^{m-1}\), and let \(P_0^+\) and \(P_0^-\) denote the distributions obtained by setting \(\theta_i=1\) and \(\theta_i=-1\), respectively.
A direct calculation gives
\[
\KL(P_0^+\Vert P_0^-)
=
\frac{\gamma}{m}\log\left(\frac{1+\gamma}{1-\gamma}\right)
\le
\frac{4\gamma^2}{m},
\]
where we used \(\log((1+t)/(1-t))\le4t\) for \(0\le t\le1/2\).
Therefore,
\[
\KL\left((P_0^+)^{\otimes n}\Vert(P_0^-)^{\otimes n}\right)\le\frac12,
\]
and Pinsker's inequality gives \(\TV((P_0^+)^{\otimes n},(P_0^-)^{\otimes n})\le1/2\).
Le Cam's two-point inequality then yields
\[
\frac12\left\{\Prob_+(\widetilde\theta_i=-1)+\Prob_-(\widetilde\theta_i=1)\right\}\ge\frac14.
\]
Averaging over \(\theta_{-i}\) shows that
\[
\frac{1}{2^m}\sum_{\theta}\Prob_\theta(\widetilde\theta_i\ne\theta_i)\ge\frac14.
\]
Consequently,
\[
\sup_{\theta}\E_\theta\TV(\widetilde P,P_0^\theta)\ge\frac{\gamma}{8}\ge\frac{1}{32\sqrt{2}}\sqrt{\frac{s_n}{n}},
\]
where we used \(m=\lfloor s_n/2\rfloor\ge s_n/4\).

Since \(\{P_0^\theta:\theta\in\{-1,1\}^m\}\subseteq\mathcal P_{s_n}\), every estimator \(\widetilde P\) satisfies
\[
\sup_{P_0\in\mathcal P_{s_n}}\E_{P_0}\TV(P_0,\widetilde P)
\ge
\frac{1}{32\sqrt{2}}\sqrt{\frac{s_n}{n}}.
\]
Taking the infimum over all estimators \(\widetilde P\) completes the proof.
\end{proof}

\section{Proofs for KL Results}
\label{sec:kl_bounds}
\subsection{Upper Bound: Proofs of Theorem~\ref{thm:kl} and Corollary~\ref{cor:kl}}

\begin{proof}[Proof of Theorem~\ref{thm:kl}]
The proof follows that of Theorem~\ref{thm:tv}, except that Lemma~\ref{lem:local-kl-transfer} replaces the triangle inequality.
Set
\[
\varepsilon_{\mathrm{KL}}:=\frac{1}{\sqrt{dMTB_0(\delta)}}.
\]
Since \(d,M,T,B_0(\delta)\ge1\), we have \(\varepsilon_{\mathrm{KL}}\in(0,1)\).
Applying Lemma~\ref{lem:grid-indexed-score-network} to \(\widehat P_0\) and arguing as in the proof of Theorem~\ref{thm:tv} gives a clipped score \(\sigma^\dagger\) satisfying
\[
\mathcal L_{\mathcal T}^{\mathrm{SE}}(\sigma^\dagger;\widehat P_0)\le\frac{dMT\varepsilon_{\mathrm{KL}}^2}{2}=\frac{1}{2B_0(\delta)}.
\]
Hence, by Lemma~\ref{lem:se_dse_equivalent} and the \(\varepsilon_{\mathrm{opt}}\)-optimality of \(\widehat\sigma\),
\[
\mathcal L_{\mathcal T}^{\mathrm{SE}}(\widehat\sigma;\widehat P_0)\lesssim\frac1{B_0(\delta)}.
\]

Applying Assumption~\ref{ass:score_based_sampler} conditionally on the observations and using the assumed sampler remainder gives
\[
\KL\left(\widehat P_\delta\Vert Q_\delta^{\mathcal A}\right)\lesssim\frac1{B_0(\delta)}.
\]
By Lemma~\ref{lem:upper_bounds}, \(\widehat P_\delta(\vx)\ge B_0(\delta)^{-1}\) for every \(\vx\).
Therefore, Lemma~\ref{lem:local-kl-transfer} yields
\[
\KL\left(P_0\Vert Q_\delta^{\mathcal A}\right)\lesssim\KL(P_0\Vert\widehat P_\delta)+\frac1{B_0(\delta)}.
\]
Since \(\delta=1/n\), Lemma~\ref{lem:kl_empirical_1/n} controls the first term.
Moreover,
\[
B_0(\delta)^{-1}=\left(\frac{1-e^{-1/n}}{K}\right)^d\le (Kn)^{-d}\lesssim\frac{\log n}{n}\le\frac{\mathfrak s_n(P_0)\log(eK^d/\mathfrak s_n(P_0))\log n}{n}.
\]
Taking expectations therefore gives
\[
\E_{P_0}\KL\left(P_0\Vert Q_\delta^{\mathcal A}\right)\lesssim\frac{\mathfrak s_n(P_0)\log\!\left(eK^d/\mathfrak s_n(P_0)\right)\log n}{n}.
\]

It remains to verify the architecture bounds.
Since \(t_1>\delta=1/n\) and \(1-e^{-1/n}\ge1/(2n)\),
\[
B_0(t_1)\le B_0(\delta)\le(2Kn)^d,
\qquad
M\le\frac{K}{1-e^{-1/n}}\le2Kn.
\]
Thus
\[
\log\left(\frac{B_0(t_1)}{\varepsilon_{\mathrm{KL}}}\right)\lesssim d\log(Kn)+\log T.
\]
Substituting this into Lemma~\ref{lem:grid-indexed-score-network} gives the stated bounds on \(L,W,S\).
Finally,
\[
B_0(t_1)^2=\left(\frac{K}{1-e^{-t_1}}\right)^{2d},
\qquad
\varepsilon_{\mathrm{KL}}^{-2}=dMT\left(\frac{K}{1-e^{-1/n}}\right)^d,
\]
which gives the stated bound on \(B\).
\end{proof}

\begin{proof}[Proof of Corollary~\ref{cor:kl}]
Define
\[
\Xi:=\frac{K}{1-e^{-1/n}}.
\]
Since \(1-e^{-1/n}\le1/n\), we have \(\Xi\ge Kn\).
First note that \(T=\log(ed\Xi^d\log K)\ge1\).

For the stated grid, \(\Delta=(d^2\Xi^d\log(Kn))^{-1}\).
By Theorem~\ref{thm:tau_leaping_sampler},
\[
\mathcal R_{\tau}(\mathcal T,\theta;M)
\lesssim e^{-T}d\log K+\Delta d\log\left(\frac{K}{\Delta}\right).
\]
Since
\[
e^{-T}d\log K=\frac{1}{e\Xi^d},
\qquad
\log\left(\frac{K}{\Delta}\right)=\log\left(d^2K\Xi^d\log(Kn)\right)\lesssim d\log(Kn),
\]
we obtain \(\mathcal R_{\tau}(\mathcal T;M)\lesssim\Xi^{-d}\), as required by Theorem~\ref{thm:kl}.

It remains to simplify the architecture bounds.
Since \(t_1=\delta+\Delta>\delta=1/n\),
\[
\frac{K}{1-e^{-t_1}}\le\Xi,
\qquad
M=B_1(t_1)\le\Xi.
\]
Moreover, \(T\lesssim d\log(Kn)\), and hence
\[
N=\left\lceil\frac{T-\delta}{\Delta}\right\rceil\lesssim d^3\Xi^d\log^2(Kn).
\]
Substituting these bounds into Theorem~\ref{thm:kl}, and using \(D_n\le n\le\Xi^d\), gives
\[
L\lesssim d^2\log^2(Kn),
\qquad
W\lesssim d^4K\Xi^d\log^2(Kn),
\qquad
S\lesssim d^4K^2\Xi^d\log^2(Kn).
\]
Finally, the bound on \(B\) in Theorem~\ref{thm:kl}, together with \(M\le\Xi\), \(T\lesssim d\log(Kn)\), and the bound on \(N\), gives $B\lesssim\Xi^{2d}.$
\end{proof}

\subsection{Minimax Lower Bound: Proof of Theorem~\ref{thm:kl_lower}}

The empirical distribution $\widehat P_0$ assigns zero probability to unobserved states and therefore has infinite KL risk.
We instead use the adaptive add-$\lambda$ estimator analyzed by \citet{mourtada_estimation_2026}, specialized to the state space $[K]^d$.

\begin{lemma}
\label{lem:kl_ad}
Define the adaptive additive-smoothing estimator
\[
\widehat P^{\mathrm{ad}}(\vx):=\frac{N_{\vx}+D_n/K^d}{n+D_n}, \qquad \vx\in[K]^d.
\]
Then
\[
\E_{P_0}\KL\left(P_0\Vert\widehat P^{\mathrm{ad}}\right)
\le 4.4\,\frac{\mathfrak s_n(P_0)\log\left(eK^d/\mathfrak s_n(P_0)\right)}{n}.
\]
\end{lemma}

\begin{proof}
For \(r>0\), define
\[
\mathfrak s_r^\circ(P_0):=\sum_{\vx\in[K]^d}\min\left\{e^{1-rP_0(\vx)},rP_0(\vx)\right\}.
\]
By \citet[Proposition~4.6]{mourtada_estimation_2026}, specialized to an alphabet of size \(K^d\),
\[
\E_{P_0}\KL\left(P_0\Vert\widehat P^{\mathrm{ad}}\right)
\le\frac{2.4\mathfrak s_n(P_0)+2\mathfrak s_{n/2}^\circ(P_0)\log\left(eK^d/\mathfrak s_{n/2}^\circ(P_0)\right)}{n+1}.
\]
For every \(u\ge0\),
\[
\min\left\{e^{1-u/2},\frac{u}{2}\right\}\le\min\{u,1\},
\]
as follows by considering \(u\le1\), \(1\le u\le2\), and \(u\ge2\).
Applying this pointwise with \(u=nP_0(\vx)\) gives \(\mathfrak s_{n/2}^\circ(P_0)\le\mathfrak s_n(P_0)\).
Since \(t\mapsto t\log(eK^d/t)\) is increasing on \((0,K^d]\) and \(\mathfrak s_n(P_0)\le K^d\),
\[
\mathfrak s_{n/2}^\circ(P_0)\log\left(\frac{eK^d}{\mathfrak s_{n/2}^\circ(P_0)}\right)\le\mathfrak s_n(P_0)\log\left(\frac{eK^d}{\mathfrak s_n(P_0)}\right),
\qquad
\mathfrak s_n(P_0)\le\mathfrak s_n(P_0)\log\left(\frac{eK^d}{\mathfrak s_n(P_0)}\right).
\]
Substituting these bounds and using \(n+1\ge n\) proves the claim.
\end{proof}

Now we prove that rate in Lemma~\ref{lem:kl_ad} is actually minimax risk.

\begin{proof}[Proof of Theorem~\ref{thm:kl_lower}]
Since \(t\mapsto t\log(eK^d/t)\) is increasing on \((0,K^d]\), the upper bound follows from Lemma~\ref{lem:kl_ad}.

For the lower bound, set
\[
m:=\min\left\{\lfloor s_n\rfloor,\left\lfloor\frac{K^d}{2}\right\rfloor\right\}.
\]
Since \(2\le s_n\le\min\{n,K^d\}\), we have \(2\le m\le n\) and \(2m\le K^d\).
Fix \(\vx_\star\in[K]^d\), and let
\[
\mathcal J:=\left\{J\subseteq[K]^d\setminus\{\vx_\star\}:|J|=m-1\right\}.
\]
For each \(J\in\mathcal J\), define \(P_0^J\) by
\[
P_0^J(\vx_\star):=1-\frac{m-1}{n},\qquad P_0^J(\vx):=\frac1n\1\{\vx\in J\},\quad \vx\ne\vx_\star.
\]
Since \(\mathfrak s_n(P_0^J)\le|\supp(P_0^J)|=m\le s_n\), every \(P_0^J\) belongs to \(\mathcal P_{s_n}\).

Draw \(J\) uniformly from \(\mathcal J\).
For any possibly randomized estimator \(\widetilde P\),
\[
\sup_{P_0\in\mathcal P_{s_n}}\E_{P_0}\KL(P_0\Vert\widetilde P)\ge \E_J\E_{P_0^J}\KL(P_0^J\Vert\widetilde P).
\]
Let \(O:=\{\vx\ne\vx_\star:N_{\vx}\ge1\}\) be the set of observed rare states.
Under \(P_0^J\), we have \(O\subseteq J\), and the likelihood under \(P_0^{J'}\) is proportional to \(\1\{O\subseteq J'\}\).
Hence, conditional on the sample, the posterior distribution of \(J\) is uniform over the sets \(J'\in\mathcal J\) containing \(O\).

Let \(\overline P:=\E[P_0^J\mid X_0^{(1)},\ldots,X_0^{(n)}]\).
Then
\[
\overline P(\vx)=
\begin{cases}
1-\frac{m-1}{n}, & \vx=\vx_\star,\\[2pt]
\frac1n, & \vx\in O,\\[2pt]
\frac1n\frac{m-1-|O|}{K^d-1-|O|}, & \vx\notin O\cup\{\vx_\star\}.
\end{cases}
\]
For every \(Q\in\Delta_{K^d}\),
\[
\E\!\left[\KL(P_0^J\Vert Q)\mid X_0^{(1)},\ldots,X_0^{(n)}\right]
=
C(X_0^{(1)},\ldots,X_0^{(n)})+\KL(\overline P\Vert Q),
\]
where \(C\) does not depend on \(Q\).
Thus the conditional Bayes risk is minimized by \(Q=\overline P\), and the same conclusion holds for randomized estimators after averaging over their auxiliary randomness.
Consequently,
\[
\E_J\E_{P_0^J}\KL(P_0^J\Vert\widetilde P)\ge\E_J\E_{P_0^J}\KL(P_0^J\Vert\overline P).
\]

If \(m-1-|O|>0\), only the unobserved states in \(J\setminus O\) contribute, and
\[
\KL(P_0^J\Vert\overline P)
=
\frac{m-1-|O|}{n}\log\left(\frac{K^d-1-|O|}{m-1-|O|}\right)
\ge
\frac{m-1-|O|}{n}\log\left(\frac{K^d}{m}\right).
\]
The same lower bound holds trivially when \(m-1-|O|=0\).
Moreover,
\[
\E_{P_0^J}[m-1-|O|]=(m-1)\left(1-\frac1n\right)^n.
\]
Therefore,
\[
\E_J\E_{P_0^J}\KL(P_0^J\Vert\overline P)
\ge
\frac{m-1}{n}\left(1-\frac1n\right)^n\log\left(\frac{K^d}{m}\right).
\]
Since \(n\ge2\), \(m\ge2\), and \(K^d\ge2m\),
\[
\left(1-\frac1n\right)^n\ge\frac14,\qquad m-1\ge\frac m2,\qquad \log\left(\frac{K^d}{m}\right)\ge\frac13\log\left(\frac{eK^d}{m}\right).
\]
Hence
\[
\E_J\E_{P_0^J}\KL(P_0^J\Vert\widetilde P)\ge\frac1{24}\frac{m\log(eK^d/m)}{n}.
\]

It remains to compare \(m\) with \(s_n\).
If \(s_n\le K^d/2\), then \(m=\lfloor s_n\rfloor\ge s_n/2\), so
\[
m\log\left(\frac{eK^d}{m}\right)\ge\frac{s_n}{2}\log\left(\frac{eK^d}{s_n}\right).
\]
If \(s_n>K^d/2\), then \(m=\lfloor K^d/2\rfloor\ge K^d/3\), while \(s_n\log(eK^d/s_n)\le K^d\), and hence
\[
m\log\left(\frac{eK^d}{m}\right)\ge\frac13s_n\log\left(\frac{eK^d}{s_n}\right).
\]
Thus, in both cases,
\[
m\log\left(\frac{eK^d}{m}\right)\ge\frac13s_n\log\left(\frac{eK^d}{s_n}\right).
\]
Since \(\widetilde P\) was arbitrary,
\[
\inf_{\widetilde P}\sup_{P_0\in\mathcal P_{s_n}}\E_{P_0}\KL(P_0\Vert\widetilde P)
\ge
\frac1{72}\frac{s_n\log(eK^d/s_n)}{n}.
\]
\end{proof}

\section{Exact One-Step Empirical Reverse Sampling}
\label{app:exact_sampler}

We construct an exact one-step empirical reverse sampler that attains the same TV and KL rates as the fixed-grid samplers in Theorems~\ref{thm:tv} and~\ref{thm:kl}, while requiring only \(O(d(D_n+K))\) operations per generated sample once the distinct observations and their multiplicities have been stored.
This shows that the unfavorable computational dependence in Corollary~\ref{cor:kl} is not intrinsic to the statistical KL rate.
Rather than discretizing reverse time or simulating intermediate jumps, the sampler draws directly from a finite-time empirical reverse kernel.

\subsection{Exact reverse kernel and sampling algorithm}
\label{app:exact_sampler_algorithm}

For \(0\le\delta<T\), define the empirical reverse transition kernel from forward time \(T\) to \(\delta\) by
\begin{equation}
\label{eq:empirical-reverse-kernel}
\overleftarrow{\mP}_{T\to\delta,\widehat P}(\vx,\vy)
:=
\frac{\widehat P_\delta(\vy)\mP_{T-\delta}(\vy,\vx)}{\widehat P_T(\vx)},
\qquad \vx,\vy\in[K]^d.
\end{equation}
By Lemma~\ref{lem:se_dse_equivalent}, the unrestricted minimizer of the exact empirical denoising score-entropy loss is the empirical score \(\widehat\sigma_t^\star(\vy,\vx)=\widehat P_t(\vy)/\widehat P_t(\vx)\).
If this score is available for every \(t\in(\delta,T]\) and the reverse dynamics are simulated exactly, the resulting reverse rates are those of the time reversal of the forward process initialized from \(\widehat P_0\).
Bayes' rule then gives \eqref{eq:empirical-reverse-kernel} as its finite-time transition kernel.
We initialize this kernel from the stationary distribution \(\pi\), and denote the resulting output distribution by
\[
Q_\delta^{\mathrm{os}}:=\pi\overleftarrow{\mP}_{T\to\delta,\widehat P}.
\]
Thus, relative to the exact empirical reverse process, the only initialization discrepancy is replacing \(\widehat P_T\) by \(\pi\).

Write \(\supp(\widehat P_0)=\{\vz_1,\ldots,\vz_{D_n}\}\), so that \(\widehat P_0(\vz_i)=N_{\vz_i}/n\).
Given \(X_T=\vx\), Bayes' rule gives
\[
w_i(\vx):=\Prob_{\widehat P_0}(X_0=\vz_i\mid X_T=\vx)
=
\frac{N_{\vz_i}\mP_T(\vz_i,\vx)}{\sum_{j=1}^{D_n}N_{\vz_j}\mP_T(\vz_j,\vx)},
\qquad i\in[D_n],
\]
where \(\mP_T(\vz_i,\vx)=\prod_{\ell=1}^d\mP_T^{\mathrm{coord}}(z_i^\ell,x^\ell)\).

Conditional on \(I=i\) and \(X_T=\vx\), the state at time \(\delta\) follows the forward bridge from \(\vz_i\) to \(\vx\).
For each coordinate \(\ell\in[d]\), define
\[
B_{i,\ell}^{T,\delta}(a\mid x^\ell)
:=
\frac{\mP_\delta^{\mathrm{coord}}(z_i^\ell,a)\mP_{T-\delta}^{\mathrm{coord}}(a,x^\ell)}
{\mP_T^{\mathrm{coord}}(z_i^\ell,x^\ell)},
\qquad a\in[K].
\]
Since the forward process evolves independently across coordinates, the bridge law factorizes as
\[
\Prob(X_\delta=\vy\mid X_0=\vz_i,X_T=\vx)
=
\prod_{\ell=1}^d B_{i,\ell}^{T,\delta}(y^\ell\mid x^\ell).
\]
Thus, after sampling \(I\), the \(K^d\)-state bridge can be sampled through \(d\) independent \(K\)-state distributions.
The complete procedure is given in Algorithm~\ref{alg:exact-empirical-reverse}.

\begin{algorithm}[t]
\caption{Exact one-step empirical reverse sampling}
\label{alg:exact-empirical-reverse}
\begin{algorithmic}[1]
\Require Times \(0\le\delta<T\), distinct observations \(\{\vz_i\}_{i=1}^{D_n}\), and counts \(\{N_{\vz_i}\}_{i=1}^{D_n}\)
\Ensure A sample \(Y\sim Q_\delta^{\mathrm{os}}\)
\State Draw \(X_T\sim\pi\)
\For{\(i=1,\ldots,D_n\)}
\State Compute \(\displaystyle \widetilde w_i\gets N_{\vz_i}\prod_{\ell=1}^d\mP_T^{\mathrm{coord}}(z_i^\ell,X_T^\ell)\)
\EndFor
\State Draw \(I\in[D_n]\) with \(\displaystyle \Prob(I=i\mid X_T)=\frac{\widetilde w_i}{\sum_{j=1}^{D_n}\widetilde w_j}\)
\For{\(\ell=1,\ldots,d\)}
\State Independently draw \(\displaystyle Y^\ell\sim B_{I,\ell}^{T,\delta}(\cdot\mid X_T^\ell)\)
\EndFor
\State \Return \(Y=(Y^1,\ldots,Y^d)\)
\end{algorithmic}
\end{algorithm}

\begin{proposition}[Correctness and computational cost]
\label{prop:one-step-correctness}
Conditional on the observations, Algorithm~\ref{alg:exact-empirical-reverse} returns an exact sample from \(Q_\delta^{\mathrm{os}}\).
If the distinct observations and their multiplicities have been stored, one execution requires \(O(d(D_n+K))\) operations.
\end{proposition}

\begin{proof}
Conditional on \(I=i\) and \(X_T=\vx\), the bridge factorization gives
\[
\Prob(Y=\vy\mid I=i,X_T=\vx)
=
\frac{\mP_\delta(\vz_i,\vy)\mP_{T-\delta}(\vy,\vx)}{\mP_T(\vz_i,\vx)}.
\]
Averaging over \(I\) yields
\[
\Prob(Y=\vy\mid X_T=\vx)
=
\frac{\widehat P_\delta(\vy)\mP_{T-\delta}(\vy,\vx)}{\widehat P_T(\vx)}
=
\overleftarrow{\mP}_{T\to\delta,\widehat P}(\vx,\vy).
\]
Averaging over \(X_T\sim\pi\) proves that \(Y\sim Q_\delta^{\mathrm{os}}\).

Computing the \(D_n\) mixture weights requires \(O(dD_n)\) operations.
Conditional on \(I\), evaluating and sampling the \(d\) coordinate bridges requires \(O(dK)\) operations.
Thus, the total cost is \(O(d(D_n+K))\).
\end{proof}

\subsection{Statistical guarantees}
\label{app:exact_sampler_guarantees}

We next show that the one-step sampler attains the TV and KL rates of Theorems~\ref{thm:tv} and~\ref{thm:kl} without constructing a fixed-grid score network.

\begin{theorem}[Statistical guarantees for the exact one-step sampler]
\label{thm:exact-one-step-guarantees}
For \(0\le\delta<T\), the following statements hold.

\medskip
\noindent\textup{(i)} \emph{TV loss.}
For every realization of the observations,
\[
\TV(P_0,Q_\delta^{\mathrm{os}})
\le
\TV(P_0,\widehat P_0)+d\delta+de^{-T}.
\]
Consequently, if \(d\delta\lesssim\sqrt{\mathfrak s_n(P_0)/n}\) and \(T\ge\log(d\sqrt{n/\mathfrak s_n(P_0)})\), then
\[
\E_{P_0}\TV(P_0,Q_\delta^{\mathrm{os}})
\lesssim
\sqrt{\frac{\mathfrak s_n(P_0)}{n}}.
\]

\medskip
\noindent\textup{(ii)} \emph{KL loss.}
Suppose additionally that \(\delta>0\).
For every realization of the observations,
\[
\KL(P_0\Vert Q_\delta^{\mathrm{os}})
\le
\KL(P_0\Vert\widehat P_\delta)
+
d\log\left(1+(K-1)e^{-T}\right).
\]
Consequently, setting \(\delta=1/n\) and choosing \(T\ge\log(Kdn/\mathfrak s_n(P_0))\) gives
\[
\E_{P_0}\KL(P_0\Vert Q_{1/n}^{\mathrm{os}})
\lesssim
\frac{\mathfrak s_n(P_0)\log\!\left(eK^d/\mathfrak s_n(P_0)\right)\log n}{n}.
\]
\end{theorem}

\begin{proof}
Since \(\widehat P_T\overleftarrow{\mP}_{T\to\delta,\widehat P}=\widehat P_\delta\), contraction of total variation under the reverse kernel gives
\[
\TV(\widehat P_\delta,Q_\delta^{\mathrm{os}})
\le
\TV(\widehat P_T,\pi)
\le de^{-T},
\]
where the last inequality follows from the coordinatewise refresh coupling.
The same coupling gives \(\TV(\widehat P_0,\widehat P_\delta)\le d\delta\).
Therefore,
\[
\TV(P_0,Q_\delta^{\mathrm{os}})
\le
\TV(P_0,\widehat P_0)+d\delta+de^{-T},
\]
and the TV conclusion follows from Lemma~\ref{lem:tv_emp}.

For the KL statement, the one-coordinate transition probabilities satisfy
\[
\mP_T^{\mathrm{coord}}(a,b)\le\frac{1+(K-1)e^{-T}}{K}.
\]
Hence
\[
\widehat P_T(\vx)\le\pi(\vx)\left(1+(K-1)e^{-T}\right)^d.
\]
Using \eqref{eq:empirical-reverse-kernel},
\[
Q_\delta^{\mathrm{os}}(\vy)
\ge
\left(1+(K-1)e^{-T}\right)^{-d}\widehat P_\delta(\vy),
\]
which implies
\[
\KL(P_0\Vert Q_\delta^{\mathrm{os}})
\le
\KL(P_0\Vert\widehat P_\delta)+d\log\left(1+(K-1)e^{-T}\right).
\]
For \(\delta=1/n\) and \(T\ge\log(Kdn/\mathfrak s_n(P_0))\), the second term is at most \(\mathfrak s_n(P_0)/n\).
The KL conclusion therefore follows from Lemma~\ref{lem:kl_empirical_1/n}.
\end{proof}

\end{document}